\documentclass[sigconf, nonacm, usenames, dvipsnames]{acmart}
\usepackage{makecell}

\usepackage{graphicx}
\usepackage{balance}  
\usepackage[linesnumbered,algoruled,boxed,lined]{algorithm2e}

\usepackage{url}

\newtheorem{problem}{Problem}

\newtheorem{definition}{Definition}

\newtheorem{lemma}{Lemma}

\newcommand{\discord}{{\textit{discord} }}

\usepackage[all]{nowidow}

\begin{document}
\title{Series2Graph: Graph-based Subsequence Anomaly Detection for Time Series}

\author{Paul Boniol}
\affiliation{%
  \institution{EDF R\&D, LIPADE, Université de Paris\\ \small{boniol.paul@gmail.com}}
}

\author{Themis Palpanas}
\affiliation{%
  \institution{LIPADE, Université de Paris \& \\ French University Institute (IUF)\\ \small{themis@mi.parisdescartes.fr}}
}

\begin{abstract}
Subsequence anomaly detection in long sequences is an important problem with applications in a wide range of domains. However, the approaches that have been proposed so far in the literature have severe limitations: they either require prior domain knowledge that is used to design the anomaly discovery algorithms, or become cumbersome and expensive to use in situations with recurrent anomalies of the same type. In this work, we address these problems, and propose an unsupervised method suitable for domain agnostic subsequence anomaly detection. Our method, Series2Graph, is based on a graph representation of a novel low-dimensionality embedding of subsequences. Series2Graph needs neither labeled instances (like supervised techniques), nor anomaly-free data (like zero-positive learning techniques), and identifies anomalies of varying lengths. The experimental results, on the largest set of synthetic and real datasets used to date, demonstrate that the proposed approach correctly identifies single and recurrent anomalies without any prior knowledge of their characteristics, outperforming by a large margin several competing approaches in accuracy, while being up to orders of magnitude faster.
This paper has appeared in VLDB 2020.
\end{abstract}


\maketitle


\section{Introduction}
\label{sec:intro}

Data series\footnote{A data series 
is an ordered sequence of real-valued points. 
If the dimension that imposes the ordering of the sequence is time then we talk about \emph{time series}, but it could also be mass (e.g., mass spectrometry), angle (e.g., astronomy), or position (e.g., biology). 
In this paper, we will use the terms \emph{time series}, \emph{data series}, and \emph{sequence} interchangeably.} anomaly detection is a crucial problem with application in a wide range of domains~\cite{itisareport,DBLP:journals/dagstuhl-reports/BagnallCPZ19}. 
Examples of such applications can be found in manufacturing, astronomy, engineering, and other domains~\cite{Palpanas:2015:DSM:2814710.2814719,itisareport},
including detection of abnormal heartbeats in cardiology~\cite{DBLP:conf/healthcom/HadjemNK16}, wear and tear in bearings of rotating machines~\cite{IMSGroundtruth},
machine degradation in manufacturing~\cite{ThemisPaper2013},
hardware and software failures in data center monitoring~\cite{DBLP:journals/pvldb/PelkonenFCHMTV15},
mechanical faults in vehicle operation monitoring~\cite{HelicopterPaperBigData2019} and
identification of transient noise in gravitational wave detectors~\cite{31ffca1e82e94ed797d33e02a8a36dff}. 
This implies a real need by relevant applications for developing methods that can accurately and efficiently achieve this goal.

\noindent{\bf [Anomaly Detection in Sequences]} 
Anomaly detection is a well studied task~\cite{statisticaloutliers, DBLP:conf/vldb/SubramaniamPPKG06, DBLP:conf/icdm/YehZUBDDSMK16,valmodjournal} that can be tackled by either examining single values, or sequences of points. 
In the specific context of sequences, which is the focus of this paper, we are interested in identifying anomalous subsequences~\cite{DBLP:conf/icdm/YehZUBDDSMK16, DBLP:conf/edbt/Senin0WOGBCF15}, which are not single abnormal values, but rather an abnormal \emph{sequence} of values. 
In real-world applications, this distinction becomes crucial: 
in certain cases, even though each individual point may be normal, the trend exhibited by the sequence of these same values may be anomalous. 
Failing to identify such situations could lead to severe problems that may only be detected when it is too late~\cite{IMSGroundtruth}.

\noindent{\bf [Limitations of Previous Approaches]} 
Some existing techniques explicitly look for a set of pre-determined types of anomalies~\cite{DBLP:conf/healthcom/HadjemNK16, dD2019604}.
These are techniques that have been specifically designed to operate in a particular setting, they require domain expertise, and cannot generalize.

Other techniques identify as anomalies the subsequences with the largest distances to their nearest neighbors (termed discords)~\cite{DBLP:conf/icdm/YehZUBDDSMK16, DBLP:conf/edbt/Senin0WOGBCF15}. 
The assumption is that the most distant subsequence is completely isolated from the "normal" subsequences. 
However, this definition fails when an anomaly repeats itself (approximately the same)~\cite{DBLP:conf/icdm/WeiKX06}. 
In this situation, 
anomalies will have other anomalies as close neighbors, and will not be identified as discords. 
In order to remedy this situation, the $m^{th} discord$ approach has been proposed~\cite{DBLP:journals/kais/YankovKR08}, which takes into account the multiplicity $m$ of the anomalous subsequences that are similar to one another, and marks as anomalies all the subsequences in the same group. 
However, this approach assumes the cardinality of the anomalies to be known, which is not true in practice (otherwise, we need to try several different $m$ values, increasing execution time).  
Furthermore, the majority of the previous approaches require prior knowledge of the anomaly length, 
and their performance 
deteriorates significantly when the correct length value is not used.

\noindent{\bf [Proposed Approach]} 
In this work, we address the aforementioned {issues}, and we propose Series2Graph, an unsupervised method suitable for domain agnostic subsequence anomaly detection. 
Our approach does not need labeled instances (like supervised techniques do), or clean data that do not contain anomalies (like zero-positive learning techniques require). 
It also allows the same model to be used for the detection of anomalies of different lengths.

Series2Graph is based on a graph representation of a novel low-dimensionality embedding of subsequences.
It starts by embedding subsequences into a vector space, where information related to their shapes is preserved. 
This space is then used to extract overlapping trajectories that correspond to recurrent patterns in the data series. 
Subsequently, we construct a graph, whose nodes are derived from the overlapping trajectories, and edges represent transitions (among subsequences in different nodes) that exist in the original series.

Intuitively, this graph encodes all the subsequences of a (single, or collection of) data series, and encodes the recurring patterns in these subsequences. 
This allows us then to differentiate between normal behavior, i.e., frequently occurring patterns, and anomalies, i.e., subsequences that rarely occur in the data series.

Overall, the experimental results (that include comparisons to several state of the art approaches, using a superset of the publicly available datasets used in the literature for subsequence anomaly detection) demonstrate that Series2Graph dominates by a large margin the competitors in accuracy, versatility, and execution time.


\noindent{\bf [Contributions]} 
Our contributions are the following.

\noindent$\bullet$ 
We propose a new formalization for the subsequence anomaly detection problem, which overcomes the shortcomings of existing models. 
Our formalization is based on the intuitive idea that anomalous are the subsequences that are not similar to the common behavior, which we call normal.

\noindent$\bullet$ 
We describe a novel low-dimensionality embedding for subsequences, and use a graph representation for these embeddings.
This representation leads to a natural distinction between recurring subsequences that constitute normal behavior, and rarely occurring subsequences that correspond to anomalies.

\noindent$\bullet$ 
Based on this representation, we develop Series2Graph~\cite{GraphAn}, an unsupervised method for domain agnostic subsequence anomaly detection. 
Series2Graph supports the identification of previously unseen single and recurring anomalies, and can be used to find anomalies of different lengths.

\noindent$\bullet$ 
Finally, we conduct an extensive evaluation using several large and diverse datasets from various domains that demonstrates the effectiveness and efficiency of Series2Graph. 



\section{Preliminaries}
\label{sec:prelim}

\noindent{\bf [Data Series]}
We begin by introducing some formal notations useful for the rest of the paper. 

	A data series $T \in \mathbb{R}^n $ is a sequence of real-valued numbers $T_i\in\mathbb{R}$ $[T_1,T_2,...,T_n]$, where $n=|T|$ is the length of $T$, and $T_i$ is the $i^{th}$ point of $T$.	
We are typically interested in local regions of the data series, known as subsequences.
	A subsequence $T_{i,\ell} \in \mathbb{R}^\ell$ of a data series $T$ is a continuous subset of the values from $T$ of length $\ell$ starting at position $i$. Formally, $T_{i,\ell} = [T_i, T_{i+1},...,T_{i+\ell-1}]$.	

Given two sequences, $A$ and $B$, of the same length, $\ell$, we can calculate their Z-normalized Euclidean distance, $dist$, as follows~\cite{DBLP:conf/kdd/ChiuKL03,DBLP:conf/sdm/MueenKZCW09,Wordrecognition,Whitney,DBLP:conf/kdd/YankovKMCZ07}: $dist = \sqrt{\sum_{i=1}^{\ell} (\frac{A_{i} - \mu_A}{\sigma_A} - \frac{B_{i} - \mu_B}{\sigma_B})^2}$, where $\mu$ and $\sigma$ represent the mean and standard deviation, respectively, of the sequences. 
In the following, we will simply use the term \emph{distance}.

Given a subsequence $T_{i,\ell}$, we say that its $m^{th}$ Nearest Neighbor ($m^{th}$ NN) is $T_{j,\ell}$ if  $T_{j,\ell}$ has the $m^{th}$ shortest distance to $T_{i,\ell}$ among all the subsequences of length $\ell$ in $T$, excluding trivial matches~\cite{DBLP:conf/icdm/ZhuZSYFMBK16}; a trivial match of $T_{i,\ell}$ is a subsequence $T_{a,\ell}$, where $|i-a| < \ell/2$ (i.e., the two subsequences overlap by more than half their length).

\noindent{\bf [Discords and their Shortcomings]}
We now define the data series \emph{discord}, which is the most prevalent subesquence anomaly definition in the literature.
\begin{definition}[Discord]\cite{DBLP:journals/kais/YankovKR08,DBLP:conf/edbt/Senin0WOGBCF15,DBLP:conf/icdm/YehZUBDDSMK16,valmodjournal}
	A subsequence $T_{i,\ell}$ is a discord if the distance between its NN, namely  $T_{j,\ell}$, is the largest among all the NN distances computed between subsequences of length $\ell$ in $T$. We require that $T_{j,\ell}$ is not a trivial match of $T_{i,\ell}$. 
\end{definition}
\begin{definition}[$m^{th}$-Discord]\cite{DBLP:journals/kais/YankovKR08}
	A subsequence $T_{i,\ell}$ is an $m^{th}$-discord if the distance between its $m^{th}$ NN, namely $T_{i,\ell}$, is the largest among all the $m^{th}$ NN distances computed between subsequences of length $\ell$ in $T$. We require that $T_{i,\ell}$ is not a trivial match of $T_{i,\ell}$.  
\end{definition}

Subsequence anomaly detection based on discords has attracted a lot of interest in the past years. 
There exist several works that have proposed fast and scalable discord discovery algorithms in various settings~\cite{DBLP:conf/edbt/Senin0WOGBCF15,Keogh2007,Liu2009,DBLP:conf/adma/FuLKL06,DBLP:conf/icdm/YehZUBDDSMK16,DBLP:conf/sdm/BuLFKPM07,DBLP:conf/icdm/YankovKR07,Parameter-Free_Discord}, including simple and $m^{th}$-discords (the authors of these papers define the problem as $k^{th}$-discord discovery), in-memory and disk-aware techniques, exact and approximate algorithms.

The strength of this definition is its mathematical simplicity. 
Nevertheless, we observe that it fails to address some challenges of real use cases. 
The reason is that large datasets may contain several anomalies that repeat themselves (approximately the same). 
In other words, it is likely that an anomaly has another anomaly as its close neighbor, and thus, does not correspond to a discord anymore.

Even though the $m^{th}$-discord definition~\cite{DBLP:journals/kais/YankovKR08} tried to address these problems, the parameter $m$ that refers to the multiplicity of some anomaly in the dataset remains a user-defined parameter, and is not easy to set correctly. 
Choosing an $m$ value that is smaller, or larger, than the correct one, will lead to both false negatives and false positives.

\begin{figure*}
  \centering
  \includegraphics[scale=0.41]{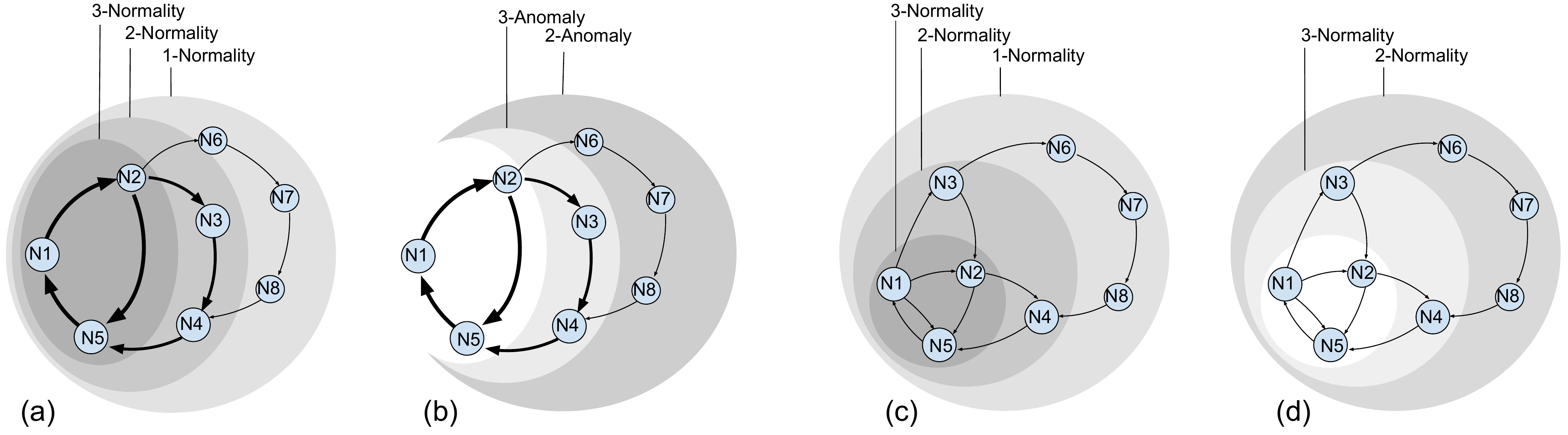}
  \caption{3-Normality, 2-Normality, 1-Normality, and 3-Anomaly, 2-Anomaly for two given graphs ((a),(b) and (c),(d)) representing the simplified model of two data series. Edge weights and node degrees are used to define the $\theta$-Normality and $\theta$-Anomaly subgraphs.}
  \label{fig:k-normality}
\end{figure*}

\noindent{\bf [Graphs]}
We introduce some basic definitions for graphs, which we will use in this paper.



We define a Node Set $\mathcal{N}$ as a set of unique integers.
Given a Node Set $\mathcal{N}$, an Edge Set $\mathcal{E}$ is then a set composed of tuples $(x_i,x_j)$, where $x_i,x_j \in \mathcal{N}$. $w(x_i,x_j)$ is the weight of that edge.

Given a Node Set $\mathcal{N}$, an Edge Set $\mathcal{E}$ (pairs of nodes in $\mathcal{N}$), a Graph $G$ is an ordered pair $G=(\mathcal{N},\mathcal{E})$.
%
A directed graph or digraph $G$ is an ordered pair $G=(\mathcal{N},\mathcal{E})$ where $\mathcal{N}$ is a Node Set, and $\mathcal{E}$ is an ordered Edge Set.

In the rest of this paper, we will only use directed graphs, denoted as $G$.

\section{Problem Formulation}
\label{sec:problem}
We now provide a new formulation for subsequence anomaly detection. 
The idea is that a data series is transformed into a sequence of abstract states (corresponding to different subsequence patterns), represented by nodes $\mathcal{N}$ in a directed graph, $G(\mathcal{N},\mathcal{E})$, where the edges $\mathcal{E}$ encode the number of times one state occurred after another. Thus, normality can be characterized by the (i) the edge weight, which indicates the number of times two subsequences occurred one after the other in the original sequence, and (ii) the node degree, the number of edges adjacent to the node, which indicates the proximity of the subsequences in that node to other subsequences. Note that G is a connected graph (there exists a path between every pair of nodes), and thus, the degree of each node is at least equal to 1. 

Under this formulation, paths in the graph composed of high-weight edges and high-degree nodes correspond to normal behavior. 
As a consequence, the Normality of a data series can be defined as follows.

\begin{definition}[$\theta$-Normality]   
Let a node set be defined as $\mathcal{N} = \{N_1,N_2,...,N_m\}$.
Let also a data series $T$ be represented as a sequence of nodes $\langle N^{(1)},N^{(2)},...,N^{(n)} \rangle$ with $\forall i \in [0,n], N^{(i)} \in \mathcal{N}$ and $m \leq n$. The $\theta$-$Normality$ of $T$ is the subgraph $G^{\nu}_{\theta}(\mathcal{N}_{\nu},\mathcal{E}_{\nu})$ of $G(\mathcal{N},\mathcal{E})$ with $\mathcal{E} = \{(N^{(i)},N^{(i+1)})\}_{i \in [0,n-1]}$, such that: $\mathcal{N}_{\nu} \subset \mathcal{N}$ and $\forall (N^{(i)},N^{(i+1)}) \in \mathcal{E}_{\nu}, w((N^{(i)},N^{(i+1)})).(deg(N^{(i)}) - 1) \geq \theta$.
\label{defNormality}  
\end{definition}

An example of $\theta$-Normality subgraph is shown in Figures~\ref{fig:k-normality}(a) and (c). 
In Figure~\ref{fig:k-normality}(a), the subgraph composed of nodes $N_1,N_2,N_5$, has edges with weights larger than $3$, and a minimum node degree of $2$. Therefore, it is a $3$-$Normality$ subgraph. 
In Figure~\ref{fig:k-normality}(c), the subgraph composed of nodes $N_1,N_2,N_5$, has edges of weight $1$, but does not have any node with a degree under $4$. 
Therefore, it is a $3$-$Normality$ subgraph.
Similarly, we define an anomaly as follows.

\begin{definition}[$\theta$-Anomaly]
Let a node set be defined as $\mathcal{N} = \{N_1,N_2,...,N_m\}$.
Let a data series $T$ be represented as a sequence of nodes $ \langle N^{(1)},N^{(2)},...,N^{(n)} \rangle $ with $\forall i \in [0,n], N^{(i)} \in \mathcal{N}$ and $m \leq n$. The $\theta$-$Anomaly$ of $T$ is the subgraph $G^{\alpha}_{\theta}(\mathcal{N}_{\alpha},\mathcal{E}_{\alpha})$ of $G(\mathcal{N},\mathcal{E})$with $\mathcal{E} = \{(N^{(i)},N^{(i+1)})\}_{i \in [0,n-1]}$, such that: $G^{\nu}_{\theta}(\mathcal{N}_{\nu},\mathcal{E}_{\nu}) \cap G^{\alpha}_{\theta}(\mathcal{N}_{\alpha},\mathcal{E}_{\alpha}) = \emptyset$.
\label{defAnomaly}  
\end{definition}

An example of $\theta$-$Anomaly$ subgraph is outlined in Figures~\ref{fig:k-normality}(b) and (d). 
In Figure~\ref{fig:k-normality}(b), the nodes that do not belong to the $3$-$Normality$ subgraph  constitute the $3$-$Anomaly$ subgraph. 
The $2$-Anomaly subgraph is included in the $3$-$Anomaly$ subgraph, and the intersection of the $2$-Anomaly and  $2$-$Normality$ subgraphs is empty. 
Similar observations hold for Figure~\ref{fig:k-normality}(d).
We now define the membership criteria of a subsequence to a $\theta$-Normality subgraph.

\begin{definition}[$\theta$-Normality Membership]   
Given a data series  $T$ represented as a sequence of abstract states $ \langle N^{(1)},N^{(2)},...,N^{(n)} \rangle $, a subsequence $T_{i,\ell}$, represented by $\langle N^{(i)},N^{(i+1)},...,N^{(i+\ell)} \rangle$, belongs to the $\theta$-$Normality$ of $T$ if and only if $\forall j \in [i,i+\ell], (N^{(j)},N^{(j+1)}) \in \theta$-$Normality(T)$. On the contrary, $T_{i,\ell}$  belongs to the $\theta$-$Anomaly$ of $T$ if and only if $\exists j \in [i,i+\ell], (N^{(j)},N^{(j+1)}) \notin \theta$-$Normality(T)$.
\label{defBelonging}  
\end{definition}

Based on the above definitions, using $\theta$-Normality subgraphs naturally leads to a ranking of subsequences based on their "normality". For practical reasons, this ranking can be transformed into a score, where each rank can be seen as a threshold in that score. We elaborate on this equivalence in the following section.
Observe also that the subsequence length is not involved in the definition of normal/abnormal, which renders this approach more general and flexible.

Note that given the existence of graph $G$, the above definitions imply a way for identifying the anomalous subsequences. 
The problem is now how to construct this graph.
Therefore, the problem we want to solve is the following.
\begin{problem}[Pattern graph construction]
	Given a data series $T$, we want to automatically construct the graph $G(\mathcal{N},\mathcal{E})$. 
\end{problem}


Following common practice, we assume that anomalies correspond to rare events.

Table~\ref{SymbolTable} summarizes the symbols we use in this paper. 

\begin{table}[tb]
\centering
\scalebox{0.95}{
\begin{tabular}{|c|c|}
\hline
{\bf Symbol} & {\bf Description} \\
\hline
$T$									& a data series \\
$|T|$									& cardinality of $T$ \\
$\lambda$ 							& convolution size\\
$\ell$ 								& input subsequence length\\
$\ell_q$ 								& query subsequence length\\
$\ell_A$ 								& anomaly length\\
$\mathcal{N}$							& set of nodes\\
$\mathcal{E}$							& set of edges\\
$G(\mathcal{N},\mathcal{E})$			& directed graph corresponding to $T$\\
$\theta$ 								& density layer (for normality/anomaly) \\
$\theta$-$Normality$						& subgraph of $G$ (also called $G^{\nu}_{\theta}$)\\
$\theta$-$Anomaly$						& subgraph of $G$ (also called $G^{\alpha}_{\theta}$)\\
$\mathcal{N}_{\nu}$						& set of nodes of $\theta$-$Normality$ \\
$\mathcal{E}_{\nu}$						& set of edges of $\theta$-$Normality$\\
$\mathcal{N}_{\alpha}$					& set of nodes of $\theta$-$Anomaly$\\
$\mathcal{E}_{\alpha}$					& set of edges of $\theta$-$Anomaly$\\
$w(e)$								& weight of edge $e \in \mathcal{E}$\\
$deg(N_i)$							& degree of node $N_i \in \mathcal{N}$\\
$Proj$								& set of all embedded subsequences \\
$Proj_r$								& reduced set $Proj$ of three dimensions\\
$SProj$								& rotated $Proj_r$\\
$\psi$ 								& angle\\
$\mathbf{\Psi}$ 						& angle set\\
$\mathcal{I}_{\psi}$						& radius set of angle $\psi$ \\
$\mathcal{N}_{\psi}$						& node set in $\mathcal{I}_{\psi}$\\
\hline
\end{tabular}
} 
\caption{Table of symbols}
\label{SymbolTable}
\end{table}

\section{Proposed Approach}
\label{sec:solution}

In this section, we describe Series2Graph, our unsupervised solution to the subsequence anomaly detection problem. 
For a given data series $T$, the overall Series2Graph process is divided into four main steps as follows (video examples of this process are available online~\cite{ourWebsite}).

\begin{enumerate}
\item {\bf Subsequence Embedding}: Project all the 
subsequences (of a given length $\ell$) of $T$ in a two-dimensional space, where 
shape similarity is preserved. 
\item {\bf Node Creation}: Create a node for each one of the densest parts of the above two-dimensional space. These nodes can be seen as a summarization of all the major patterns of length $\ell$ that occurred in $T$. 
\item {\bf Edge Creation}: Retrieve all transitions between pairs of subsequences represented by two different nodes: each transition corresponds to a pair of subsequences, where one occurs immediately after the other in the input data series $T$. 
We represent transitions with an edge between the corresponding nodes. 
The weights of the edges are set to the number of times the corresponding pair of subsequences was observed in $T$.
\item {\bf Subsequence Scoring}: Compute the normality (or anomaly) score of a subsequence of length $\ell_q \geq \ell$ (within or outside of $T$), based on the previously computed edges/nodes and their weights/degrees.
\end{enumerate}

We note that length $\ell$ required in the first step of the method is user defined, but is independent of the length of the subsequences that we want to detect as anomalies, which can have different lengths.
Moreover, the proposed approach is robust to the choice of $\ell$, especially when the $\ell$ value is larger than the length of the targeted anomalies.
In contrast, existing methods require prior knowledge of the anomaly length and can only discover anomalies of that length.
We demonstrate these points in the experimental evaluation.
 
Below, we describe in detail each one of the above steps.



\subsection{Subsequence Embedding}

We first describe our approach for projecting a data series into a two-dimensional space. 
We propose a new shape-based embedding, such that two subsequences similar in shape will be geometrically close in the transformed space after the embedding. 
In order to achieve this, we (i) extract all the subsequences and represent them as vectors, (ii) reduce the dimensionality of these vectors to three dimensions (that we can visualize in a three-dimensional space), (iii) rotate the space of these vectors (i.e., subsequences) such that two of the components contain the shape related characteristic, and the last one the average value. 
As a result, two subsequences with similar shape, but very different mean value (i.e., small Z-normalized Euclidean distance, but large Euclidean distance) will have very close values for the first two components, but very different for the third one.


We start by extracting subsequences using a sliding window that we slide by one point at a time\footnote{This is equivalent to using a {\it time delay embedding}~\cite{Kantz:2003:NTS:1121581} with a delay $\tau=1$.}, and then applying a local convolution (of size\footnote{We use $\lambda=\ell / 3$ in all our experiments. Varying $\lambda$ between $\ell / 10$ and $\ell / 2$ does not affect the results (we omit this graph for brevity).} $\lambda=\ell / 3$) to each subsequence, in order to reduce noise and highlight the important shape information. 
Formally, for a given subsequence length $\ell$ and local convolution size $\lambda=\ell / 3$, we transform subsequence $T_{i,\ell}$ into a vector (of size $\ell - \lambda$):
\[\bigg[  \sum_{k=i}^{i+\lambda} T_k,\sum_{k=i+1}^{i+1+\lambda} T_k, ... ,\sum_{k=i+\ell-\lambda}^{i+\ell} T_k \bigg]
\]

We insert the vectors corresponding to all subsequences $T_{i,\ell}$ 
in matrix $Proj(T,\ell,\lambda) \in \mathbb{M}_{|T|,\ell-\lambda}(\mathbb{R})$, where $\mathbb{M}$ is the set of real-valued matrices with $|T|$ rows and $\ell-\lambda$ columns.

In order to reduce the dimensionality of matrix $Proj(T,\ell,\lambda)$, we apply a Principal Component Analysis (PCA) transform.
For the sake of simplicity, we keep only the first three components ($PCA_3$), and denote the reduced three-column matrix as $Proj_r(T,\ell,\lambda)$.

We note that using the first three components was sufficient for our study. 
Consider that for the $25$ datasets used in our experimental evaluation, the three most important components explain on average $95$\% of the total variance. 
Generalizing our solution to a larger number of important components is part of our current work.

\begin{figure*}
  \centering
  \includegraphics[scale=0.65]{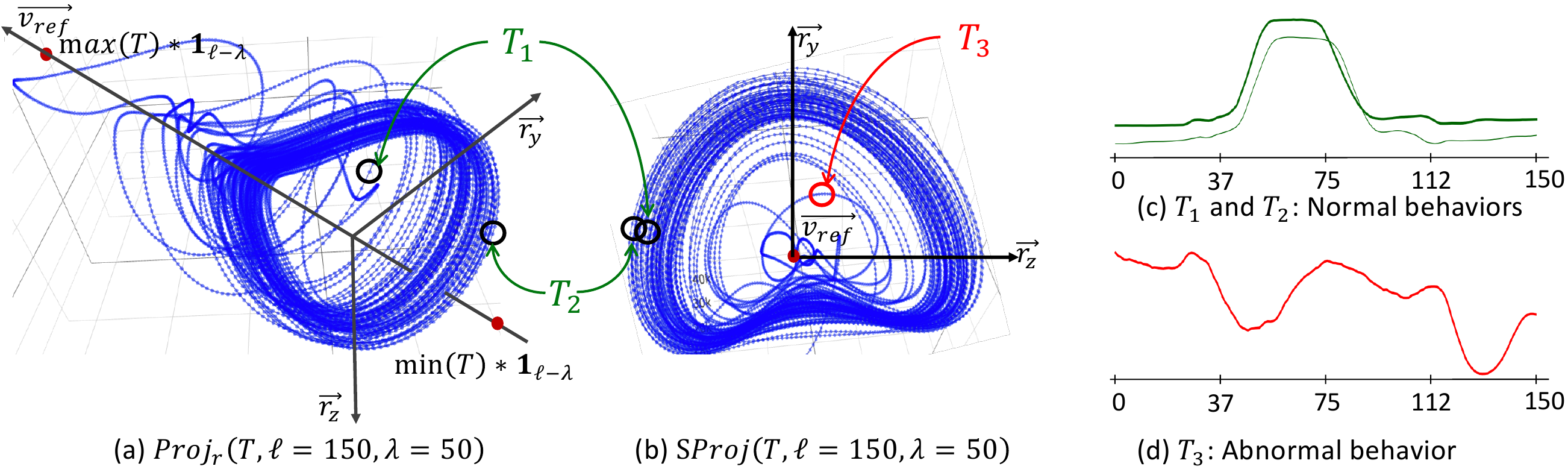}
  \caption{(a) $Proj_r(T,\ell,\lambda)$ and (b) $SProj(T,\ell,\lambda)$ of a data series $T$ corresponding to the movement of an actor's hand that (c) takes a gun out of the holster and points to a target (normal behavior); (d) the anomaly (red subsequence) corresponds to a moment when the actor missed the holster~\protect\cite{Keogh:2004:TPD:1014052.1014077}. We rotate (a) into (b) such that $\vec{v}_{ref}$ is invariant in two dimensions.}
  \label{fig:rotationExample}
\end{figure*}

%

Since we are interested in subsequence anomalies, which correspond to anomalous shapes (trends), we need to emphasize (out of the three components obtained by the aforementioned reduced projection) the components that explain the most the shape of the subsequences. 
Let $min(T)$ and $max(T)$ be the minimum and maximum values of the data series $T$. 
We extract the vector $\vec{v}_{ref} = \overrightarrow{O_{mn}O_{mx}}$, where $O_{mn} = PCA_3(min(T)*\lambda*\mathbf{1}_{\ell-\lambda})$ and $O_{mx} = PCA_3(max(T)*\lambda*\mathbf{1}_{\ell-\lambda})$ ($PCA_3$ returns the three most important components using the trained PCA applied on $Proj(T,\ell,\lambda)$).
Intuitively, the vector $\vec{v}_{ref}$ describes the time dimension along which the values change (bounded by $\lambda*min(T)$ and $\lambda*max(T)$, where the multiplication with $\lambda$ corresponds to a local convolution). The other dimensions (orthogonal vectors of $\vec{v}_{ref}$) describe how the values change. Thus, overlapping points/sequences in these other dimensions indicate recurrent behaviors and isolated points/sequences indicate possible anomalies.
Given the unit vectors $(\vec{u}_x,\vec{u}_y,\vec{u}_z)$ that represent the axes of the cartesian coordinate system of the PCA,
the angle $\phi_x = \angle \vec{u}_x \vec{v}_{ref}$, $\phi_y = \angle \vec{u}_y \vec{v}_{ref}$, $\phi_z = \angle \vec{u}_z \vec{v}_{ref}$ and their corresponding rotation matrices $R_{u_x}(\phi_x)$,$R_{u_y}(\phi_y)$ and $R_{u_z}(\phi_z)$,  we define $SProj(T,\ell,\lambda)$ as follows: 
$SProj(T,\ell,\lambda) = R_{u_x}(\phi_x)R_{u_y}(\phi_y)R_{u_z}(\phi_z)Proj_r(T,\ell,\lambda)^T$.
The matrix $SProj(T,\ell,\lambda)$ is the reduced projection $Proj_r(T,\ell,\lambda)$ rotated in order to have the unit vector $\vec{u}_x$ aligned to the offset vector $\vec{v}_{ref}$. 


Figure~\ref{fig:rotationExample} depicts the rotation procedure to transform $Proj_r$ into $SProj$ for an example data series $T$ that corresponds to the movement of an actor's hand that takes a gun out of the holster and points to a target (normal behavior).
This rotation is using vector $\vec{v}_{ref}$, defined by the minimal and maximal constant sequences mentioned earlier (marked with the red dots in Figure~\ref{fig:rotationExample}(a)). 
The unit vectors of the rotated space are $(\frac{\vec{v}_{ref}}{||\vec{v}_{ref}||}, \vec{r}_{y},\vec{r}_{z})$, where $\vec{r}_{y}$ and $\vec{r}_{z}$ are the rotated vectors $\vec{u}_{y}$ and $\vec{u}_{z}$. 

What this rotation achieves is that (similarly to Z-normalization) subsequences with a different mean but the same shape in the space before the rotation (e.g., subsequences $T_1$ and $T_2$ in Figure~\ref{fig:rotationExample}(c)) will have very close $\vec{r}_{y}$ and $\vec{r}_{z}$ components in the new coordinate system (as shown in Figures~\ref{fig:rotationExample}(a) and (b)).
Therefore, subsequences with similar shapes will appear close together, shapes that repeat often in the dataset will form dense clusters in the space (like subsequences $T_1$ and $T_2$), and rare shapes (anomalies) will appear relatively isolated (like subsequence $T_3$).
Figures~\ref{fig:rotationExample}(c) and (d) depict the normal ($T_1$ and $T_2$) and abnormal ($T_3$) subsequences.
The anomaly ($T_3$) corresponds to a case when the actor missed the holster~\cite{Keogh:2004:TPD:1014052.1014077}.

We observe that in the rotated space (see Figure~\ref{fig:rotationExample}(b)), the shape differences are easy to distinguish, and the identification of normal behavior (dense clusters of repeated patterns) and anomalies (isolated patterns) is visually obvious. 

In the rest of the paper, $SProj(T,\ell,\lambda)$ will refer to the 2-dimensions matrix keeping only the $r_{y}$ and $r_z$ components.


\IncMargin{0.5em}
\begin{algorithm}[tb]
{\small
    \caption{\textbf{Pattern Embedding}}\label{alg:PatternEmbedding}
    \SetKwInOut{Input}{input}
    \SetKwInOut{Output}{output}
    \Input{Data series $T$, input length $\ell$, $\lambda$}
    \Output{3-dimension points sequence $SProj$}
    \BlankLine
    \tcp{Transform first subsequence}
    $P$ $\leftarrow$ $\bigg(  \sum_{k=j}^{j+\lambda} T_k    \bigg)_{j \in [0,\ell - \lambda]}$\;
    add $P$ in $Proj$\;
    \tcp{Transform every other subsequences in T}
    \ForEach{$i \in [1,|T| - \ell]$ }
    {
        $P[0:\ell - \lambda - 1]$ $\leftarrow$ $P[1:\ell - \lambda]$\; \label{storePreviousRes}
        $P[\ell - \lambda]$ $\leftarrow$ $ \sum_{k=i+\ell - \lambda}^{i+\ell} T_k$\;
        add $P$ in $Proj$\;
    }
    \tcp{Reduce to three dimensions}
    $pca$ $\leftarrow$ $PCA_3.fit(Proj)$\; \label{PCA}
    $Proj$ $\leftarrow$ $pca.transform(Proj)$\; 
	\tcp{Get rotation characteristics}    
    $v_{ref}$ $\leftarrow$ $pca.transform((max(T) - min(T))*\lambda*\mathbf{1}_{\ell-\lambda})$\;
    $\phi_x$,$\phi_y$,$\phi_z$ $\leftarrow$  $getAngle(({u}_x,{u}_y,{u}_z),{v}_{ref})$\;
    $R_{u_x}$,$R_{u_y}$,$R_{u_z}$ $\leftarrow$  $GetRotationMatrices(\phi_x,\phi_y,\phi_z)$\; \label{RotationMatrices}
	\tcp{Rotate SProj}    
    $SProj$ $\leftarrow$ $R_{u_x}.R_{u_y}.R_{u_z}.Proj^T$
} 
\end{algorithm}
\DecMargin{0.5em}

Algorithm~\ref{alg:PatternEmbedding} describes the computation of the pattern embeddings. 
A naive solution is to compute all the convolutions for all the subsequences of $T$, which leads to a complexity of magnitude $O(|T|\ell\lambda)$.
Nevertheless, by using the previously computed convolutions (Line~\ref{storePreviousRes}), the complexity is reduced to $O(|T|\lambda)$. 
The PCA step of Algorithm~\ref{alg:PatternEmbedding} is implemented with a randomized truncated Singular Value Decomposition (SVD), using the method of Halko et al.~\cite{SVDHALKO} with a time complexity of $O(|T|(\ell-\lambda)|component|)$. 
The last step consists of matrix multiplications, and therefore has a complexity of $O(|T||R_{u_x}|^2)$. 
In our case, the size of the rotation matrices are much smaller than $\lambda$, which leads to a global complexity of $O(3|T|(\ell - \lambda))$.

\subsection{Node Creation}

At this point, we are ready to extract shape related information, as in Figure \ref{fig:rotationExample}, where recurrent and isolated trajectories can be distinguished. 
The idea is to extract the most {\it crossed} sections of the 2-dimensional space defined by the unit vector $(\vec{r}_{y},\vec{r}_{z})$.  
These sections will be the nodes in the graph we want to construct. 

First, we define the \emph{radius subset}.
\begin{definition}[Radius Set]
   Given a data series $T$ and its projection matrix $P=SProj(T,\ell,\lambda)$, the radius set $\mathcal{I}_{\psi}$ is the set of intersection points between the vector $\vec{u}_{\psi} = cos(\psi)\vec{r}_y + sin(\psi)\vec{r}_z$ and every segment $[x_{i-1},x_{i}]$, where $x_{i-1},x_{i}$ are two consecutive rows of P: 
    $\mathcal{I}_{\psi} = \big\{ x \big| (\vec{u}_{\psi} \times \vec{x} = \vec{0}) \land (\overrightarrow{x_{i-1}x} \times \overrightarrow{x_{i-1}x_{i}} = \vec{0})  \big\}$, where $\times$ operator is the cross product.
    \label{defAnomalies} 
\end{definition}

\begin{figure}[tb]
	\centering
	\includegraphics[scale=0.60]{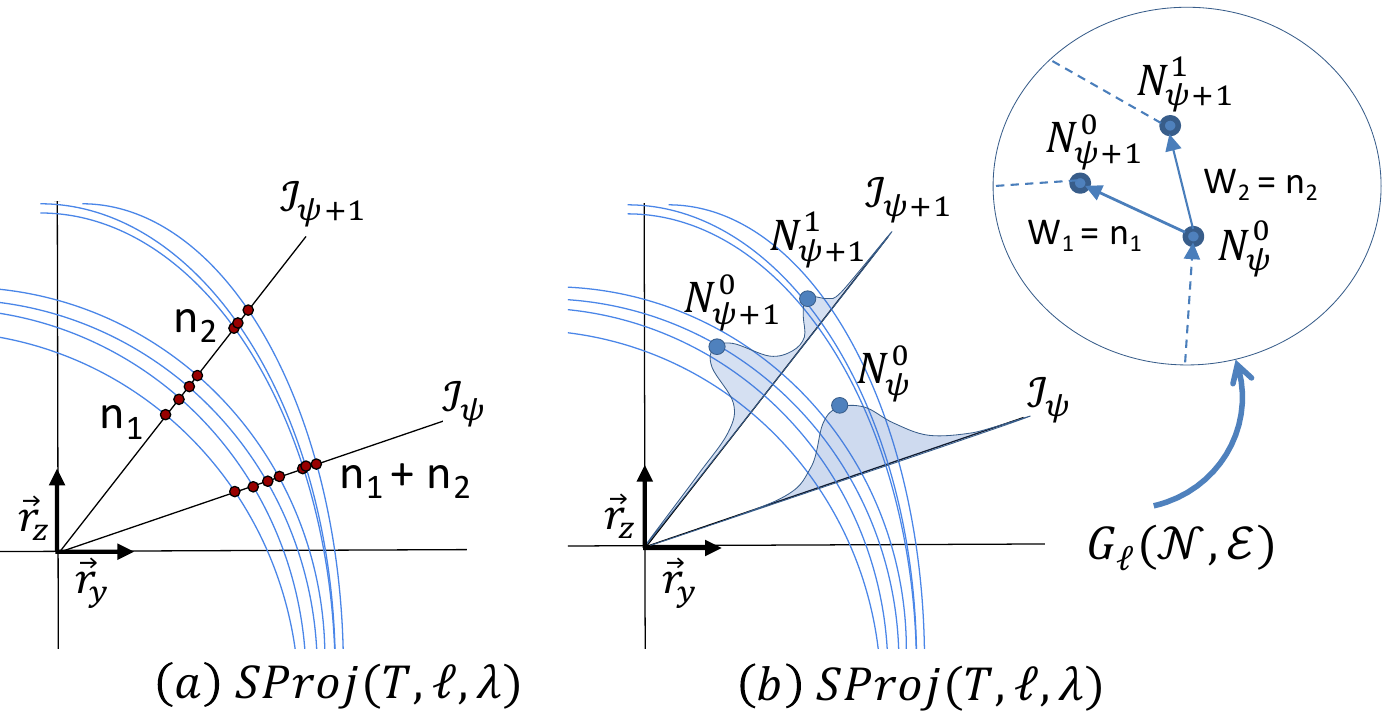}
	\caption{Node extraction from $SProj(T,\ell,\lambda)$ by measuring the density of the intersected trajectories to a given vector. The densest points are added to the Node Set $\mathcal{N}$.}
	\label{fig:Node_extraction}
\end{figure}

Figure \ref{fig:Node_extraction} (a) displays two radius subsets (marked with the red points). 
We can now define the \emph{Pattern Node Set} as follows.

\begin{definition}[Pattern Node Set]
    Given a data series $T$, its projection $P=SProj(T,\ell,\lambda)$ and a set of $\mathcal{I}_{\psi}$ ($\psi \in \mathbf{\Psi}$), the Pattern Node Set of T is:
    \begin{align*}
    \mathcal{N} &= \cup_{\psi \in \mathbf{\Psi}} \mathcal{N}_{\psi} \\
    \mathcal{N}_{\psi} &= \big\{ x \big| \exists \epsilon, \forall y, |x-y| > \epsilon \implies f_h(x,\mathcal{I}_{\psi}) > f_h(y,\mathcal{I}_{\psi}) \big\} \\
    \text{with } &f_h(x,\mathcal{I}_{\psi}) = \frac{1}{nh\sqrt{2 \pi \sigma(\mathcal{I}_{\psi})^2}} \sum_{x_i \in \mathcal{I}_{\psi}} e^{\frac{(x-x_i - h\mu(\mathcal{I}_{\psi}))^2}{2h\sigma(\mathcal{I}_{\psi})^2}}
    \end{align*}
    \label{defAnomalies} 
\end{definition}
In the above definition, $f$ is a kernel density estimation function, applied on a radius subset using a Gaussian distribution. 
Then, nodes become the areas in the 2-dimensional space, where the trajectories of the patterns are the most likely to pass through. 
In other words, each node corresponds to a set of very similar patterns.
The bandwidth parameter $h$ affects the granularity of the extraction: the smaller the $h$ value is, the more local maxima, and therefore the more nodes we will end up with; the larger the $h$ value is, the less nodes the graph will have, and therefore the more general it will be.
We define parameter $r=|\mathbf{\Psi}|$ as the number of angles $\psi$ that we use in order to extract the pattern node set. In other words, this parameter is sampling the space (refer to Algorithm~\ref{alg:NodeExtraction}, Line~\ref{RS_noa}).
Once again, a big number of angles will lead to high precision, but at the cost of increased computation time. 

In practice, we observed that parameter $r$ is not critical, and we thus set $r=50$ for the rest of this paper. 
Regarding the bandwidth parameter of the density estimation, we set it following the Scott's rule~\cite{scott:1992}:
$h_{scott} = \sigma(\mathcal{I}_{\psi}) . |\mathcal{I}_{\psi}|^{-\frac{1}{5}}$.

\IncMargin{0.5em}
\begin{algorithm}[tb]
	{\small
    \caption{\textbf{Node Extraction}}\label{alg:NodeExtraction}
    \SetKwInOut{Input}{input}
    \SetKwInOut{Output}{output}
    \Input{2-dimensional point sequence $SProj$, rate $r$, bandwidth $h$ }
    \Output{Node Set $\mathcal{N}$}
    \BlankLine
    \tcp{Set the number of radius}
    $\mathbf{\Psi}$ $\leftarrow$ $\big( i\frac{2\pi}{r} \big)_{i \in [0,r]}$\; \label{RS_noa}
    $\mathcal{N}$ $\leftarrow$ $\{\}$\;
    \ForEach{$\psi \in \mathbf{\Psi}$ }
    {
        $\mathcal{I}_{\psi}$ $\leftarrow$ $[]$\;
        \ForEach{$i \in [0,|SProj|-1]$ }
        {
        	\tcp{Find intersected points}
            $radius$ $\leftarrow$ $max_{x,y}(SProj_{i},SProj_{i+1})$\; \label{RS_fill_b}
            $P_{\psi}$ $\leftarrow$ $(radius_x.cos(\psi),radius_y.sin(\psi))$\;
            add $Intersect((\Omega, P_{\psi}),(SProj_{i},SProj_{i+1}))$ in $\mathcal{I}_{\psi}$ \label{RS_fill_e}
        }
        \tcp{Extract Nodes}
        $\mathcal{N}_{\psi}$ $\leftarrow$  $argmax_{x \in \mathcal{I}_{\psi}} f_h(x,\mathcal{I}_{\psi})$\;
        add $\mathcal{N}_{\psi}$ in $\mathcal{N}$\;
    }
} 
\end{algorithm}
\DecMargin{0.5em}

Algorithm~\ref{alg:NodeExtraction} outlines the above process for extracting the node set from $SProj$. 
In Lines~\ref{RS_fill_b}-\ref{RS_fill_e}, we compute for each radius subset $\mathcal{I}_{\psi}$ all intersection points between a radius vector and the possible segments composed of two consecutive points in $SProj$. 
The complexity of this operation is bounded by $O(|SProj|r) \simeq O(|T|r)$. 
The time complexity of the kernel density estimation is $O(|\mathcal{I}_{\psi}|)$ (since $|\mathcal{I}_{\psi}| \leq |SProj|$). 
Actually, we experimentally observed that $|\mathcal{I}_{\psi}| << |SProj|$.
Therefore, the overall time complexity is bounded by $O(|T|r)$. 
We can improve this complexity using the following observation. 
Instead of checking the intersection with every possible radius, we can select those that bound the position of the points $i$ and $i+1$ in $SProj$ (only the radius with $\psi$ between $\psi_i = \angle \vec{u_x}. \vec{SProj_{i}}$ and $\psi_{i+1}=\angle \vec{u_x}. \vec{SProj_{i+1}}$).
Therefore, the worst case complexity becomes $O(|T|r)$, and the best case complexity is reduced to $O(|T|)$.

\subsection{Edge Creation}

Once we identify the nodes, we need to compute the edges among them. 
Recall that the set of extracted nodes corresponds to all the possible states, where subsequences of the data series $T$ can be. 
In order to compute the transitions among these nodes, we loop through the entire projection $SProj(T,\ell,\lambda)$ and we extract the sequence $\langle N^{(0)},N^{(1)}, ... , N^{(n)} \rangle$ of the nodes $N_i$ in $\mathcal{N}$ that the embedded subsequences $(SProj(T,\ell,\lambda)_0,SProj(T,\ell,\lambda)_1, ... ,SProj(T,\ell,\lambda)_n)$ belong to.
Intuitively, the above node sequence involves all the nodes in $\mathcal{N}$ (some of them more than once) and represents the entire input data series. 
We use this sequence to identify the edges of the graph $G_{\ell}$ we want to construct.
In practice, we extract the edges (all the pairs of successive nodes in the above sequence) and set their weights to the number of times the edge is observed in the sequence.
Formally, the edges set $\mathcal{E}$ is defined as follows.

\begin{definition}[Pattern Edges Set]
    Given a data series $T$, its projection $P=SProj(T,\ell,\lambda)$ and its Pattern Node Set $\mathcal{N}$, the edges set $\mathcal{E}$ is equal to:
    $\mathcal{E} = \big\{ \big( S(P_i),S(P_{i+1} ) \big) \big\}_{i \in [1,|P|-1]}$,
    where  function $S$ finds for a given projection point, the closest node in $\mathcal{N}$. Formally:
    $S(x) = argmin_{n \in \mathcal{N}} d(x,n) $, where $x \in P$ and $d$ is the geometrical Euclidean distance.
    \label{defAnomalies} 
\end{definition}

Since the weight of each edge is equal to the cardinality of this edge in the edge set $\mathcal{E}$, this weight is proportional to the number of times two subsequences follow each other in the input data series. 
For efficiency, $S(x)$ is computed as follows: for a given projection point, we first find the node subset $N_{\psi}$ of $\mathcal{N}$ (with $\psi \in \mathbf{\Psi}$), such that $|\angle \vec{x} \vec{u_{\psi}}|$ is minimal. We then compute $S(x)$ such as $S(x) = argmin_{n \in \mathcal{N}_{\psi}} |x.\vec{u_{\psi}} - n |$, where $\vec{x}.\vec{u_{\psi}}$ is the scalar product between $\vec{x}$ and $\vec{u_{\psi}}$.
As depicted in Figures~\ref{fig:Node_extraction}(a) and (b), a total of $n_1+n_2$ subsequences are intersected by $\mathcal{I}_{\psi}$ and represented by node $N_{\psi}^{0}$. 
At $\mathcal{I}_{\psi +1}$, these subsequences are divided between nodes $N_{\psi+1}^{0}$ ($n_1$ subsequences) and $N_{\psi+1}^{1}$ ($n_2$ subsequences). 
Therefore, we have $w(N_{\psi}^{0},N_{\psi+1}^{0})=n_1$ and $w(N_{\psi}^{0},N_{\psi+1}^{1})=n_2$.

\IncMargin{0.5em}
\begin{algorithm}[tb]
	{\small
    \caption{\textbf{Edge Extraction}}\label{alg:EdgeExtraction}
    \SetKwInOut{Input}{input}
    \SetKwInOut{Output}{output}
    \Input{2-dimension points sequence $SProj$, and a node set $\mathcal{N}$ }
    \Output{a Edge Set $\mathcal{E}$}
    \BlankLine
    $\mathbf{\Psi}$ $\leftarrow$ $\big( i\frac{2\pi}{r} \big)_{i \in [0,r]}$\;
    $NodeSeq$ $\leftarrow$ []\; 
    \ForEach{$i \in [0,|SProj|-1]$ }
    {
    	\tcp{Get the radius that bound SProj$_{i}$ and SProj$_{i+1}$}
        $\psi_i$ $\leftarrow$ $getAngle(\vec{u_x},SProj_{i})$\;   
        $\psi_{i+1}$ $\leftarrow$ $getAngle(\vec{u_x},SProj_{i+1})$\; 
        \ForEach{$(\psi \in \mathbf{\Psi}) \land (\psi \in [\psi_i,\psi_{i+1}])$ }
        {
        	\tcp{Fill the sequence of node NodeSeq}
            $radius$ $\leftarrow$ $max_{x,y}(SProj_{i},SProj_{i+1})$\;
            $P_{\psi}$ $\leftarrow$ $(radius_x.cos(\psi),radius_y.sin(\psi))$\;
            $x_{int}$ $\leftarrow$ $Intersect((\Omega, P_{\psi}),(SProj_{i},SProj_{i+1}))$\; 
            $n_{int}$ $\leftarrow$ $argmin_{n \in \mathcal{N}_{\psi}}(|x_{int} - n|)$\; 
            add $n_{int}$ in $NodeSeq$
    }
    \tcp{Extract edges from NodeSeq}
    $\mathcal{E}$ $\leftarrow$  $\big\{ (NodeSeq_{i},NodeSeq_{i+1}) \big\}_{i \in [0,|fullPath|]}$\;
    }
}
\end{algorithm}
\DecMargin{0.5em}

Algorithm~\ref{alg:EdgeExtraction} outlines the steps we follow to extract the edges among the nodes in $\mathcal{N}$. 
For each point in the input data series $T$, we identify the radius it belongs to and we choose the closest node. 
Therefore, the complexity is bounded by $O(|T|)$ and varies based on the number of radius we have to check and the number of nodes in each $\mathcal{N}_{\psi}$. 
The former is bounded by parameter $r$: on average, we have no more than $|T|/r$ points per $\mathcal{N}_{\psi}$. 
The overall complexity is in the worst case $O(|T|^2)$, and in the best case $O(|T|)$. 
We note that this worst case corresponds to the situation where each subsequence in $T$ belongs to a different node. 
This is not what we observe in practice: for all our datasets, the overall complexity is close to the best case. 


\subsection{Subsequence Scoring}

We now describe how we can use the information in the graph to identify the normal and anomalous behaviors.


\noindent{\bf [Subsequence to Path Conversion]}
We start with the conversion of a subsequence to a path in a given graph. 
For an already computed graph $G_{\ell}(\mathcal{N},\mathcal{E})$, we define function $Time2Path(G_{\ell},T_{i,\ell_q})$ that converts a subsequence $T_{i,\ell_q}$ into a path (i.e., a sequence of nodes) in $G_{\ell}$, by (i) computing the pattern embedding $SP$ of $T_{i,\ell_q}$ (using the PCA transformation and rotation matrices, computed in Lines~\ref{PCA} and~\ref{RotationMatrices}, respectively, of Algorithm~\ref{alg:PatternEmbedding}), and (ii) extracting the edges using $EdgeExtraction(SP,\mathcal{N})$ (output of Algorithm~\ref{alg:EdgeExtraction}) on the node set $\mathcal{N}$ of graph $G_{\ell}$.



\noindent{\bf [Normality Score]}
We are now ready to measure normality. 
As mentioned earlier, modeling a data series using a cyclic graph results in the graph encoding information concerning the recurrence of subsequences. 
Then, the path normality score function can be defined as follows.
\begin{definition}[Path Normality Score]
    Given a data series $T$ and its graph $G_{\ell}(\mathcal{N},\mathcal{E})$, and a subsequence $T_{i,\ell_q}$ of length $\ell_{q} \geq \ell$, the normality of a path $P_{th} = Time2Path(G_{\ell}(\mathcal{N},\mathcal{E}),T_{i,\ell_q}) = \langle N^{(i)},N^{(i+1)}, ..., N^{(i+\ell_q)} \rangle $ is equal to:
    $Norm(P_{th}) = \sum_{j=i}^{i+\ell_q-1} \frac{w(N^{(j)},N^{(j+1)})(deg(N^{(j)})-1) }{\ell_q}$
    \label{defNormalPath} 
\end{definition}

We can thus infer a normality score for subsequences in $T$ using the $Time2Path$ function defined earlier (the opposite of this normality score is the anomaly score). 
Formally, the normality score is defined as follows.
\begin{definition}[Subsequence Normality Score]
    Given a data series $T$, its graph $G_{\ell}(\mathcal{N},\mathcal{E})$ and a subsequence $T_{i,\ell_q}$ of length $\ell_q \geq \ell$, the $Normality$ score $T_{i,\ell_q}$ is equal to:
    $Normality(T_{i,\ell_q}) = Norm(Time2Path(G_{\ell}(\mathcal{N},\mathcal{E}),T_{i,\ell_q}))$
    \label{defNormalScore} 
\end{definition}

Observe that the two previous definitions are consistent with the definition of $\theta$-Normality, such that every $P_{th}$ in $\theta$-Normal subgraph will have $N(P_{th}) \geq \theta$, and every $P_{th}$ that is exclusively in a lower normality level will have $N(P_{th}) \leq \theta$. 
As a matter of fact, the rank generated by the normality score is similar to the $\theta$-$Normality$ rank, and in both rankings, the anomalies are found at the bottom of the ranking. 
The following lemma formalizes this statement.

\begin{lemma}
    Given a data series $T$, its graph $G_{\ell}(\mathcal{N},\mathcal{E})$, a subsequence $T_{i,\ell_q}$, and its path $P_{th}=Time2Path(G_{\ell}(\mathcal{N},\mathcal{E}),T_{i,\ell_q})$, we have:
    $\forall \theta \in \mathbb{N}_{>0},  N(P_{th}) < \theta \implies P_{th} \in \theta\text{-Anomaly}(T)$
    \label{defAnomalies} 
\end{lemma} 

\begin{proof}
Consider a subsequence $T_{i,\ell_q}$ corresponding to $P_{th} = \langle N^{(i)},N^{(i+1)}, ..., N^{(i+\ell_q)} \rangle $. If $P_{th}\in \theta\text{-Normality}(T)$, then according to Definition \ref{defBelonging}, we have: 
\begin{align*}
&\forall j \in [i,i+\ell_q-1], (N^{(j)},N^{(j+1)})  \in \theta\text{-Normality}(T)\\
\implies &\forall j \in [i,i+\ell_q-1], w(N^{(j)},N^{(j+1)}).(deg(N^{(j)})-1)  \geq \theta \\
\implies &\sum_{j=i}^{i+\ell_q-1} \frac{w(N^{(j)},N^{(j+1)}).(deg(N^{(j)}) - 1) }{\ell_q}  \geq \theta
\end{align*}
As a consequence, $P_{th} \notin \theta\text{-Normality}(T)$ and according to Definition \ref{defBelonging}, $P_{th} \in \theta\text{-Anomaly}(T)$.
\end{proof}

Therefore, the subsequences of $T$ with a low score are those that compose the {\it $\theta$-Anomaly} subgraph, where the value of $\theta$ is low (close to one for the discords).
We note that this process identifies both single anomalies (discords) and recurrent anomalies.


\subsection{Series2Graph Summary}




%





Algorithm~\ref{alg:ScoringAlgo} summarizes all the steps of our approach. 
In Lines~\ref{embeddingLine}-\ref{EdgeSetLine}, we compute the subsequence embedding, the node set $\mathcal{N}$ and then the edge set $\mathcal{E}$ in order to build the graph $G_{\ell}$. 
Line~\ref{NormalityScoreLine} computes the $Normality$ score for all subsequences of the input data series: we use a sliding window over the input data series to extract all subsequences, we score each one of them and store the result in the vector  $NormalityScore=[Normality(T_{0,\ell_q}), ... ,Normality(T_{n - \ell_q,\ell_q})]$, initialized in Line~\ref{NormalityScoreInit}. 
Finally, 
we apply a moving average filter 
on the $NormalityScore$ vector (Line~\ref{exp_step}).
This filter tries to rectify possible small inaccuracies of the scoring function, by ensuring that two highly overlapping subsequences will have similar $Normality$ scores (as we would normally expect).


\begin{algorithm}[tb]
	{\small
	\caption{\textbf{Series2Graph}}\label{alg:ScoringAlgo}
	\SetKwInOut{Input}{input}
	\SetKwInOut{Output}{output}
	\Input{data series $T$, input length $\ell$, query length $\ell_q$}
	\Output{a data series $AnomScore$}
	\BlankLine
	$SProj$ $\leftarrow$ $PatternEmbedding(T,\ell,\lambda)$\; \label{embeddingLine}
	$\mathcal{N}$ $\leftarrow$ $NodeExtraction(SProj,r=50,h=h_{opt})$\;
	$\mathcal{E}$ $\leftarrow$ $EdgeExtraction(SProj,\mathcal{N})$\; \label{EdgeSetLine}
	$G_{\ell}$ $\leftarrow$ $Graph(\mathcal{N},\mathcal{E})$\;
	\BlankLine
	\tcp{vector of Normality scores}
	$NormalityScore$ $\leftarrow$ $[0]_{0,|T| - \ell_q}$\; \label{NormalityScoreInit}
	\tcp{compute Normality score for all subsequences of length $\ell_q$ in T}
	\ForEach{$i \in [1,|T| - \ell_q]$ }
	{
		$NormalityScore[i]$ $\leftarrow$ $Norm(Time2Path(G_{\ell},T_{i,\ell_q}))$\; \label{NormalityScoreLine}
	}
	$NormalityScore$ $\leftarrow$ $movingAverage(NormalityScore,\ell)$\; \label{exp_step}
} 
\end{algorithm}

\section{Experimental Evaluation}
\label{sec:exp}

We now present the results of the experimental evaluation with several real datasets from different domains, including all the annotated datasets that have been used in the discord discovery literature. 
In order to ensure the reproducibility of our experiments, we created a web page~\cite{ourWebsite} with the source code, datasets, and other supporting material. 

\subsection{Experimental setup and Datasets}

We have implemented our algorithms~\cite{GraphAn} in C (compiled with gcc 5.4.0). 
The evaluation was conducted on a server with Intel Xeon CPU E5-2650 v4 2.20GHz, and 250GB of RAM. 

We benchmark our approach using different annotated real and synthetic datasets, listed in Table~\ref{datasetInfo}.
Following previous work~\cite{DBLP:journals/tkdd/SeninLWOGBCF18}, we use several synthetic datasets that contain sinusoid patterns at fixed frequency, on top of a random walk trend. 
We then inject different numbers of anomalies, in the form of sinusoid waveforms with different phases and higher than normal frequencies, and add various levels of Gaussian noise on top. 
We refer to those datasets using the label SRW-[\# of anomalies]-[\% of noise]-[length of anomaly] and use them in order to test the performance of the algorithms under different, controlled conditions. 

Our real datasets are:
(i) Simulated engine disks data (SED) collected by the Rotary Dynamics Laboratory at NASA~\cite{doi:10.1177/1475921710395811, doi:10.1117/12.847574}. 
This data series represents disk revolutions recorded over several runs. 
(ii) MIT-BIH Supraventricular Arrhythmia Database (MBA)~\cite{Goldbergere215,Moody}, which are electrocardiogram recordings from 5 different patients, containing multiple instances of two different kinds of anomalies.
(iii) Five additional real datasets from various domains that have been studied in earlier works~\protect\cite{DBLP:conf/icdm/KeoghLF05,DBLP:conf/edbt/Senin0WOGBCF15,DBLP:conf/edbt/Senin0WOGBCF15}, and their anomalies are simple \emph{discords} (usually only 1): aerospace engineering (Space Shuttle Marotta Valve~\cite{DBLP:conf/icdm/KeoghLF05}), gesture recognition (Ann's Gun dataset~\cite{DBLP:conf/edbt/Senin0WOGBCF15}) and medicine (patient's respiration measured by the thorax extension~\cite{DBLP:conf/icdm/KeoghLF05}, and the record 15 of the BIDMC Congestive Heart Failure Database~\cite{DBLP:conf/icdm/KeoghLF05}).

We measure $Top$-$k$ accuracy (i.e., the correctly identified anomalies among the $k$ retrieved by the algorithm, divided by $k$) and execution time.

\begin{table}[tb]
	\centering
	\scalebox{0.90}{
\begin{tabular}{|c|c|c|c|c|}
\hline
{\bf Datasets} & {\bf Length} & $\mathbf{\ell_A}$ & $\mathbf{N_A}$ & {\bf Domain } \\
\hline
\hline
SED			 	& 100K	&75 &  50 & Electronic\\
MBA (803) 		& 100K	&75 &  62 & Cardiology\\
MBA (804) 		& 100K	&75 &  30 & Cardiology\\
MBA (805)		& 100K	&75 &  133 & Cardiology\\
MBA (806)		& 100K	&75 &  27 & Cardiology\\
MBA (820)		& 100K	&75 &  76 & Cardiology\\
MBA (14046)		& 100K	&75 &  142 & Cardiology\\
Marotta Valve 		& 20K &1K &  1 & \makecell{Aerospace\\engineering}\\
Ann Gun 			& 11K & 800 &  1 &  \makecell{Gesture\\recognition}\\
Patient Respiration 	& 24K & 800 &  1 & Medicine\\
BIDMC CHF 	& 15K & 256 &  1 & Cardiology\\
\makecell{SRW-[20-100]-[0\%]-[200]} & 100K & 200 & var. & Synthetic \\
\makecell{SRW-[60]-[5\%-25\%]-[200]}& 100K & 200 & 60 & Synthetic \\
\makecell{SRW-[60]-[0\%]-[100-1600]}& 100K & var. & 60 & Synthetic \\
\hline
\end{tabular}
}
\caption{List of dataset characteristics: data series length, anomaly length ($\ell_A$), number of annotated anomalies ($N_A$) and domain.}
\label{datasetInfo}
\end{table}


\subsection{Length Flexibility}

We first evaluate the influence of the subsequence length parameters $\ell$ and $\ell_q$. 
Ideally, we would like the method to be robust to variations in these parameters. 
This will ensure that the constructed model can accurately identify anomalies, even if these anomalies are of several lengths, significantly different than what was originally expected by the users.
We stress that, in contrast, all previous techniques~\cite{DBLP:conf/edbt/Senin0WOGBCF15,DBLP:conf/icdm/YehZUBDDSMK16,DBLP:conf/icdm/YankovKR07,Breunig:2000:LID:342009.335388,LSTManomaly} require knowledge of the \emph{exact} anomaly length, and are very brittle otherwise. 

Previous approaches developed for discord discovery, like $STOMP$\cite{DBLP:conf/icdm/YehZUBDDSMK16}, are very sensitive to even slight variations of the subsequence length parameter (for the identification of the anomalous subsequences). 
Figure~\ref{fig:multi_length_example} displays the Euclidean distances of each subsequence of the MBA ECG records 803 to its nearest neighbor (computed using STOMP), using lengths 80 and 90, given the length of the anomaly $\ell_A=80$. 
The results demonstrate that a small variation in the input length can lead to very different outcomes (using a length equal to 90, the discord is a normal heart beat, and therefore a false positive).

\begin{figure}[tb]
  \centering
  \includegraphics[scale=0.52]{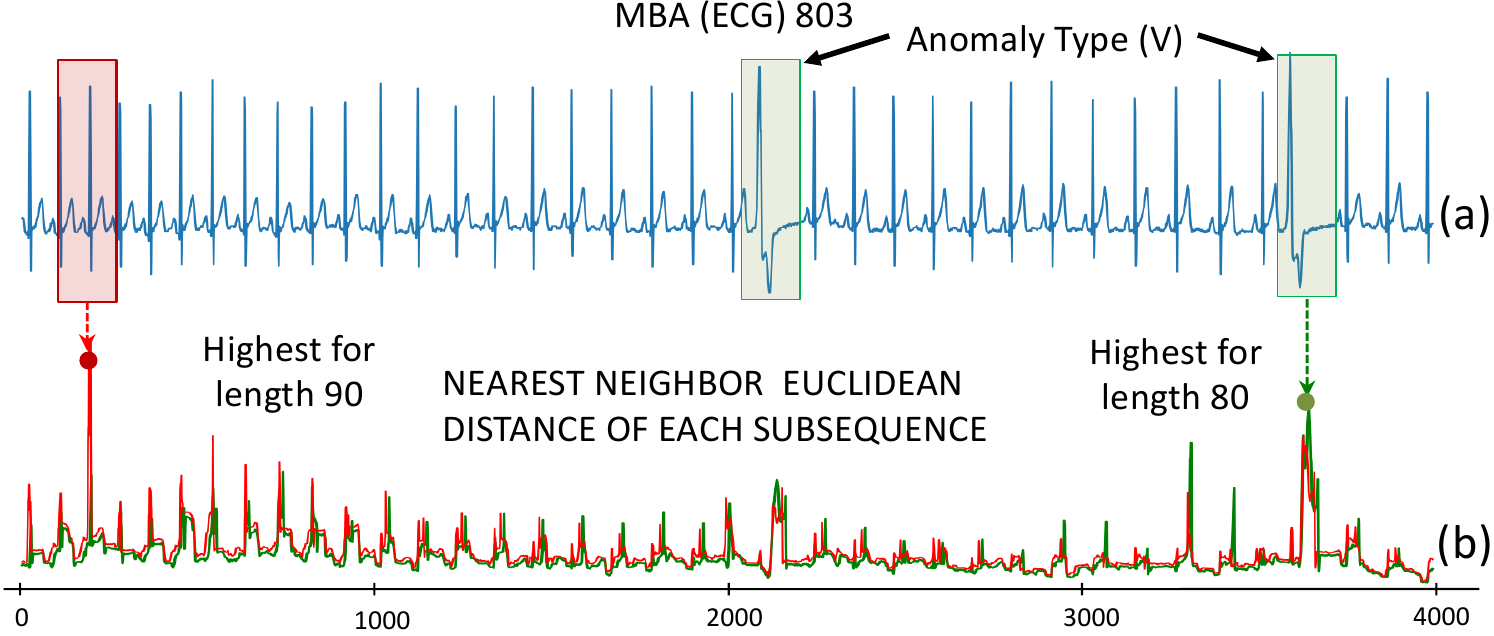}
  \caption{(a) MBA ECG recording (\textit{4000} points snippet from patient 803), with one annotated Supraventricular contraction (V). (b) Euclidean distances of each subsequence (for the entire MBA ECG recording) of length \textit{80} (green) and \textit{90} (red). In this two cases, the subsequences with the highest distances are not the same (length \textit{90}: Normal Beat, length \textit{80}: Anomaly Type V).}
  \label{fig:multi_length_example}
\end{figure}

On the other hand, Series2Graph is robust to variations in the query subsequence length.
Figure~\ref{fig:multi_length_method} depicts the $G_{\ell}(\mathcal{N},\mathcal{E})$ graphs for $\ell$ equals to $80,100,120$, while the anomalies length is $75$ for Type S anomalies, and $120$ for type V anomalies. 
The results show that in all cases, irrespective of the length $\ell$ used to construct the graph, the anomaly trajectories (Type V highlighted in red and Type S highlighted in blue) are distinct from the highly-weighted trajectories (thick black) that correspond to normal behavior. 

\begin{figure}[tb]
	\centering
	\includegraphics[scale=0.30]{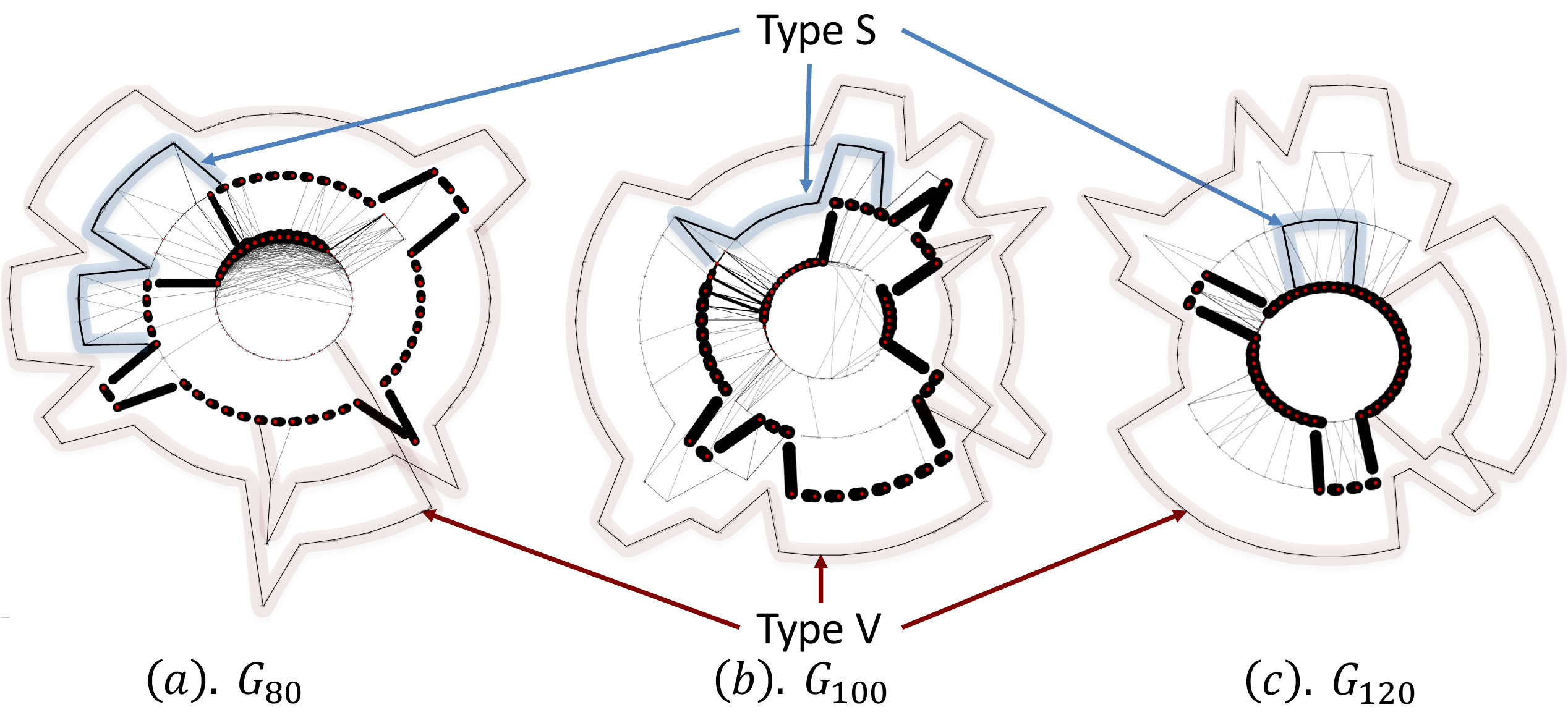}
	\caption{$G_{\ell}(\mathcal{N},\mathcal{E})$ of the MBA(820) electrocardiogram data series for $\ell$ of $80,100$ and $120$. In the three cases, the different kinds of anomalies (S: blue, V: red) are well separable with lower edges weights.}
	\label{fig:multi_length_method}
\end{figure}

In order to confirm this observation, we conduct several experiments. 
First, we measure the $Top$-$k$ accuracy as the input length $\ell$ and the query length $\ell_q$ vary (using a query length $2\ell_q/3 = \ell$, with the anomaly length $\ell_A=80$). 
Figure~\ref{fig:multi_length_exp_2}(a) demonstrates the stable behavior of Series2Graph. 
Even though the $Top$-$k$ accuracy varies for small lengths, 
the performance remains relatively stable when the input lengths $\ell$ we use to construct the graph are larger than the anomaly length $\ell_A$.
This means that simply selecting an $\ell$ value larger than the expected anomaly length $\ell_A$ will lead to good performance.

\begin{figure*}[tb]
  \centering
  \includegraphics[scale=0.66]{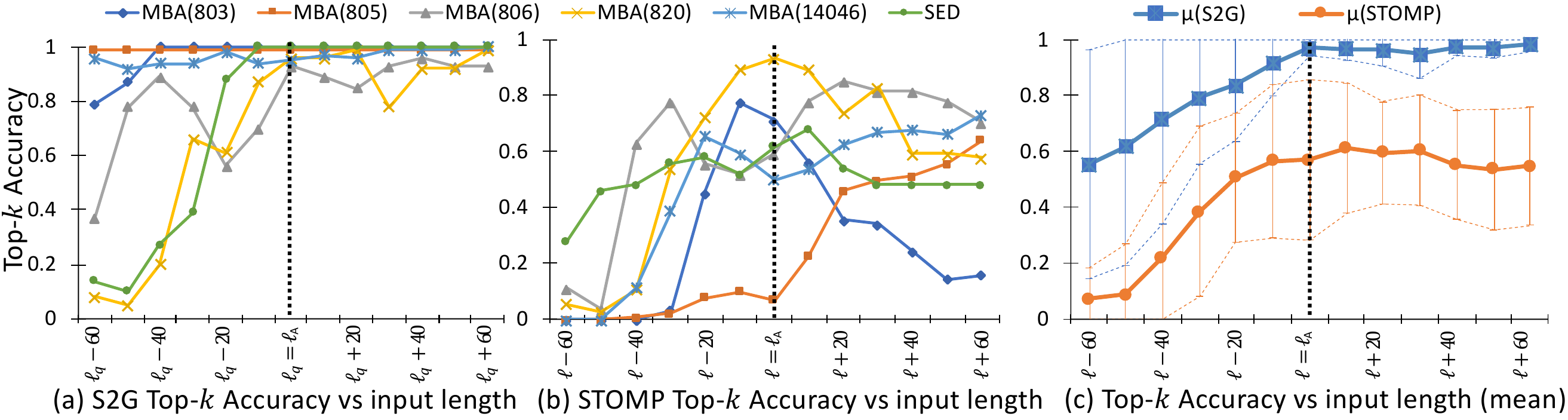}
  \vspace*{-0.3cm}
  \caption{On the MBA and SED datasets: (a) $Top$-$k$ accuracy of Series2Graph varying the input and query length ($2\ell_q/3 = \ell$) to build the graph $G_{\ell}$. (b) STOMP $Top$-$k$ accuracy varying the input length $\ell$, (c) STOMP and Series2Graph in $Top$-$k$ accuracy average compared to the input length $\ell$.}
  \label{fig:multi_length_exp_2}
\end{figure*}

In contrast, as Figure~\ref{fig:multi_length_exp_2}(b) demonstrates, the performance of STOMP (a discord-based approach) varies widely. 
Thus, such approaches need careful tuning, requiring domain expertise and good knowledge of the possible anomaly lengths.
Furthermore, even though STOMP accuracy seems to converge to a stable value as the length is increasing, the Series2Graph accuracy stays significantly higher and much more stable in average, as shown in Figure~\ref{fig:multi_length_exp_2}(c).




\subsection{Optimal Bandwidth}

\begin{figure*}[tb]
	\centering
	\includegraphics[scale=0.66]{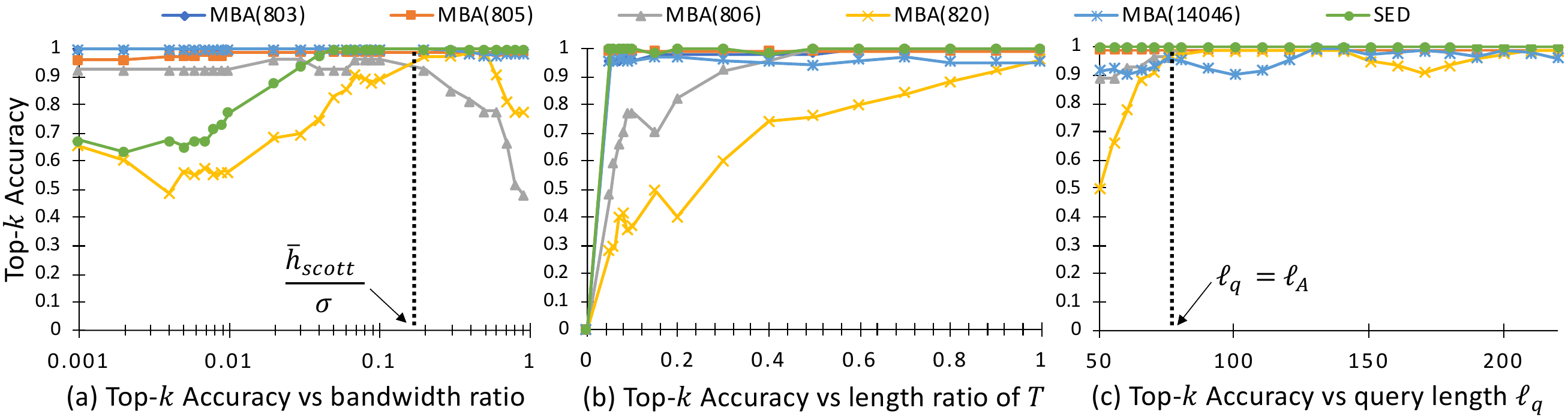}
	\vspace*{-0.4cm}
	\caption{$Top$-$k$ accuracy of Series2Graph on the MBA and SED datasets while: (a) varying the bandwidth ratio $h/\sigma(\mathcal{I}_{\psi})$ in $f_h(x,\mathcal{I})$ (logarithmic scale), (b) varying the length of the prefix snippet used to build the graph, (c) varying the query length $\ell_q$ (starting from the input length $\ell$) to compute the normality score.}
	\label{fig:multi_length_exp_1}
\end{figure*}

We now evaluate the effectiveness of the kernel bandwidth $h_{scott}$ in $f_h(x,\mathcal{I})$ in the node extraction step. 
We set a constant value for $\ell$ and $\ell_q$ ($\ell=80$, $\ell_q=160$), and we measure the accuracy for different bandwidths. 
Figure~\ref{fig:multi_length_exp_1}(a) displays the $Top$-$k$ accuracy on the MBA and SED datasets as a function of $h/\sigma(\mathcal{I}_{\theta})$ (logarithmic scale). 
As expected, a small bandwidth ratio breaks down too much the normal pattern, and therefore reduces its Normality score, while a large bandwidth ratio (above 0.7) hinders some key nodes to detect anomalies in two of the six datasets, namely MBA(806) and MBA(820). 
The anomalies in these two datasets are close to the normal behavior, thus the abnormal trajectories can be easily missed. 
In contrast, using the Scott bandwidth ratio $h_{scott}$ (marked with the dotted line) leads to very good accuracy for all the datasets we tested (We used the datasets with the same anomaly and pattern lengths, so that we can compare Scott bandwidth ratios).

\subsection{Convergence of Edge Set}

In this section, we evaluate Series2Graph accuracy on previously unseen data series (i.e., on different data series than the one used to build the graph $G$, with potentially different anomalies). 
Note that the Normality score of a non-existing pattern in $G$ will be $0$ (see Definitions~\ref{defNormalPath} and~\ref{defNormalScore}) and therefore close to the score of the anomalies 
in $G$.




In the experiment of Figure~\ref{fig:multi_length_exp_1}(b), we build the graph using only a prefix of the input series. 
We then vary the size of this prefix and measure the $Top$-$k$ accuracy for the entire series.
As we can see, the $Top$-$k$ accuracy usually reaches its maximum value without having to build the graph using all the available data. 
We observe that on average $Top$-$k$ accuracy already reaches more than 85\% of its maximum value, when we use as little as 40\% of the input data series. 
Nevertheless, one can also see that the $Top$-$k$ accuracy of MBA(820) and MBA(806) converge slower than the other datasets. 
These two datasets contain anomalies of Type S, which means they are very similar to a normal heartbeat. 
Therefore, Series2Graph requires more subsequences in order to build a model that effectively separates the anomalous behavior from the normal behavior. 

Finally, Figure~\ref{fig:multi_length_exp_1}(c) demonstrates the stability of Series2Graph in the anomaly discovery task as we vary the query subsequence length $\ell_q$ using a fixed input length $\ell$. 
The results show that we can identify anomalies with very high accuracy, by scoring candidate subsequences of a wide range of lengths, provided they are larger than the length of the anomalies ($\ell_q\ge\ell_A$).

\subsection{Discord Identification}

\begin{figure*}[tb]
	\centering
	\includegraphics[scale=0.49]{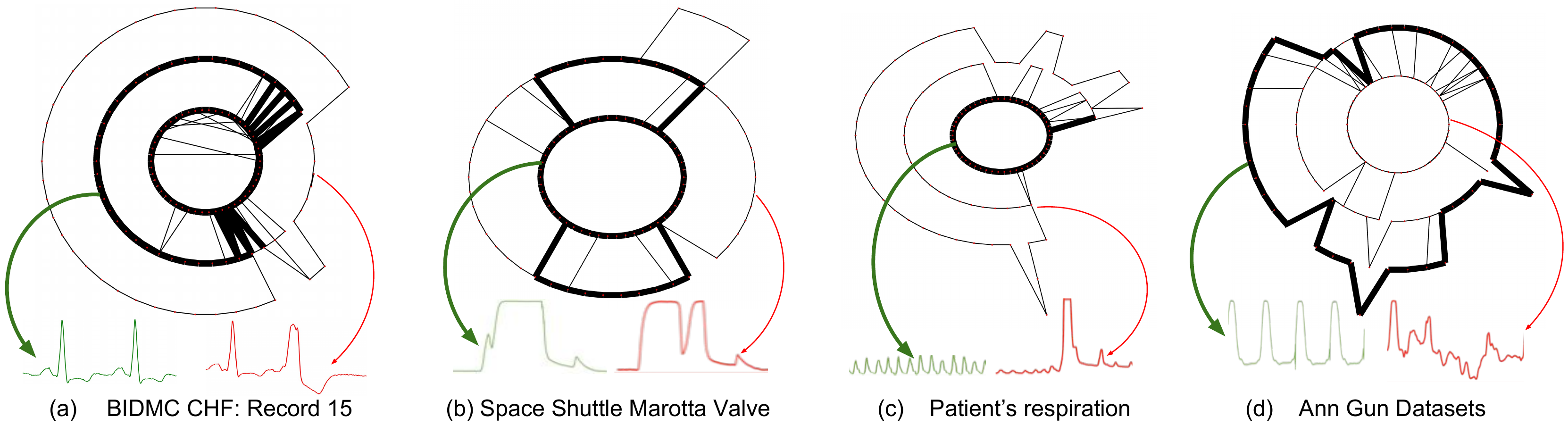}
	\vspace*{-0.4cm}
	\caption{(a) $G_{80}$ for BIDMC Congestive Heart Failure Database: Record 15, (b) $G_{200}$ for Space Shuttle Marotta Valve (TEK 16), (c) $G_{50}$ Patient's respiration, (d) $G_{150}$ Ann Gun Datasets. Green subsequences belong to $\theta$-Normality subgraph with large $\theta$. Red subsequences belong to $\theta$-Anomaly subgraph with small $\theta$.}
	\label{fig:discord_finding}
\end{figure*}

The next experiments evaluate the capability of Series2Graph to identify discords. 
We used the datasets that have appeared in the discord discovery literature (mentioned above). 
Figure~\ref{fig:discord_finding} shows the graphs obtained for Ann Gun, Marotta Valve, Patient's respiration, and Record 15 of the BIDMC Congestive Heart Failure Database. The thickness of the lines correspond to the weights of the edges.
This figure shows that the discords of these datasets (red subsequences) always correspond to trajectories with low weights ($\theta$-Anomaly, for a small $\theta$ value), whereas the normal subsequences (green subsequences) correspond to trajectories with high weights ($\theta$-Normality, for a large $\theta$). 
Therefore, the anomaly scores of the discords are in every case the largest, and hence are correctly identified.

\subsection{Accuracy Evaluation}

In this section, we report the anomaly detection accuracy results. 
We compare Series2Graph to the current state-of-the-art data series subsequence anomaly detection algorithms, using $\ell_q = \ell_A$. 
For Series2Graph, we use the same $\ell=50$ and latent $\lambda=16$ for all datasets, even though different values would be optimal in every case, thus demonstrating the robustness of the approach. 
We also evaluate Series2Graph's learning ability by comparing the accuracy obtained by building the graph using only the first half of the data series ($S2G_{|T|/2}$), compared to using the full data series ($S2G_{|T|}$).
We consider two techniques that enumerate $Top$-$k$ $1^{st}$ discords, namely GrammarViz (GV)~\cite{DBLP:conf/edbt/Senin0WOGBCF15} and STOMP~\cite{DBLP:conf/icdm/YehZUBDDSMK16}. 
Moreover, we compare Series2Graph against the Disk Aware Discord Discovery algorithm (DAD)~\cite{DBLP:conf/icdm/YankovKR07}, which finds $m^{th}$ discords. 
We also compare to Local Outlier Factor (LOF)~\cite{Breunig:2000:LID:342009.335388} and Isolation Forest~\cite{Liu:2008:IF:1510528.1511387}. 
Note that the last two methods are not specific to subsequence anomaly discovery. 
Finally, we include in our comparison LSTM-AD~\cite{LSTManomaly}, a \emph{supervised} deep learning technique. 
We stress that the comparison to LSTM-AD is not fair to all the other techniques: LSTM-AD has to first train on \emph{labeled} data, which gives it an unfair advantage; all the other techniques are \emph{unsupervised}.
We include it to get an indication as to how the unsupervised techniques compare to a state-of-the-art supervised anomaly detection algorithm.

\begin{table*}[tb]
\centering
\hspace*{-0.3cm}
\scalebox{1.1}{
\begin{tabular}{|c|c|c|c|c|c|c|c|c|}
\hline
\textbf{Datasets} & \textbf{GV} & \textbf{STOMP} &  \textbf{DAD} & \textbf{LOF} & \textbf{Isolation Forest (std)} & \textbf{LSTM-AD} & \textbf{$S2G_{|T|/2}$} & \textbf{$S2G_{|T|}$} \\
\hline
SED			 		& 0.46 		& 0.73 		& 0.44 		& 0.65 		& 0.65 (0.02) & 0.10			& {\bf 1.00} 	& {\bf 1.00} \\
MBA (803) 			& 0.15 		& 0.60		& 0.01 		& 0.08 		& {\bf 1.00 (0.00)} & 0.35 			& {\bf 1.00}	& {\bf 1.00} \\
MBA (805)			& 0.09 		& 0.10 		& 0.03  	& 0.42 		& 0.99 (0.01) & 0.85 			& 0.99 			& {\bf 0.99} \\
MBA (806)			& 0.01 		& 0.59 		& 0.66 		& 0.92 & 0.75 (0.06) & 0.10			& 0.96 			& {\bf 1.00} \\
MBA (820)			& 0.05 		& 0.92 		& 0.04  	& 0.42 		& 0.92 (0.03) & 0.09			& 0.76 			& {\bf 0.96} \\
MBA (14046)			& 0.09 		& 0.54 		& 0.71  	& 0.64 		& 0.99 (0.01) & {\bf 1.00}	& 0.95	& 0.95 \\
\hline

SRW-[20]-[0\%]-[200] 	& {\bf 1.0}	& 0.77 		& 0.55  	& 0.74	&  0.75 (0.05) & 0.94		& 0.95			& {\bf 1.00} \\
SRW-[40]-[0\%]-[200] 	& 0.975  	& {\bf1.0} 	& 0.05 		& 0.89 	& 0.92 (0.02) & {\bf1.00}	& {\bf 1.00}	& {\bf 1.00} \\
SRW-[60]-[0\%]-[200] 	& 0.96  	& 0.88  	& 0.10 		& 0.76 	& 0.87 (0.02) & 0.92		& {\bf 1.00}	& {\bf 1.00} \\
SRW-[80]-[0\%]-[200] 	& 0.96 		& 0.43  	&  0.14  	& 0.82 	& 0.86 (0.01) & 0.95		& 0.98			& {\bf 1.00} \\
SRW-[100]-[0\%]-[200] 	& 0.95 		& 0.99  	& 0.11   	& 0.75 	& 0.92 (0.02) & {\bf1.00}	& {\bf 1.00}	& {\bf 1.00} \\
\hline

SRW-[60]-[5\%]-[200] 	& {\bf 1.0}	& 0.73  	& 0.21 		& 0.88 & 0.89 (0.01) & 0.96 		& {\bf 1.00}	& {\bf 1.00} \\
SRW-[60]-[10\%]-[200] 	& 0.83 		& {\bf0.98}	& 0.01  	& 0.70 & 0.80 (0.01) & 0.94		& 0.96			& {\bf 0.98} \\
SRW-[60]-[15\%]-[200] 	& 0.76  	& 0.62  	& 0.17  	& 0.66 & 0.82 (0.01) & 0.94		& {\bf 0.98}	& {\bf 0.98} \\
SRW-[60]-[20\%]-[200] 	& 0.73  	& {\bf1.0}	& 0.01  	& 0.73 & 0.85 (0.02) & 0.96		& {\bf 1.00}	& {\bf 1.00} \\
SRW-[60]-[25\%]-[200] 	& 0.63  	& 0.64 		& 0.09  	& 0.67 & 0.80 (0.01) & 0.83		& {\bf 0.98}	& {\bf 0.98} \\
\hline

SRW-[60]-[0\%]-[100] 	& 0.98  	& {\bf 1.0}	& 0.23 		& 0.74 & 0.88 (0.02) & {\bf1.00} & 0.96	 		& 0.96  \\
SRW-[60]-[0\%]-[200] 	& 0.96  	& 0.60 		& 0.19  	& 0.85 & 0.83 (0.01) & {\bf1.00}	& 0.98			& 0.98 \\
SRW-[60]-[0\%]-[400] 	& 0.98  	& {\bf 1.0}	& 0.63 		& 0.76 & 0.88 (0.01) & 0.88 		& 0.96			& 0.96 \\
SRW-[60]-[0\%]-[800] 	& 0.91  	& 0.86 		& - 		& 0.69 & 0.87 (0.01) & 0.76 		& 0.95			& {\bf 0.98} \\
SRW-[60]-[0\%]-[1600] 	& {\bf 1.0}	& {\bf 1.0}	& -   		& 0.52 & 0.64 (0.02) & 0.90 		& 0.91			& 0.94 \\
\hline
\textbf{Average} 		& 0.62 		& 0.73 		& 0.24  	& 0.68 & 0.85 & 0.78 		& 0.96 			& {\bf 0.98} \\
\hline
\end{tabular}
} 
\caption{$Top$-$k$ accuracy for DAD, STOMP, GrammarViz, LSTM-AD, $S2G_{|T|/2}$ (Series2Graph built using half of the dataset) and $S2G_{|T|}$ (Series2Graph built using the entire dataset) with $k$ equal to number of anomalies.}
\label{scalabilityAccuracy}
\end{table*}

In Table~\ref{scalabilityAccuracy}, we show the $Top$-$k$ accuracy. 
We set $k$ equal to the number of anomalies.
These experiments test the capability of each method to correctly retrieve the $k$ anomalous subsequences in each dataset. 
For Series2Graph, we simply have to report the $Top$-$k$ anomalies that Algorithm~\ref{alg:ScoringAlgo} produces. 
For the \discord based techniques, we have to consider the $Top$-$k$ $1^{st}$ \discord and the $m^{th}$ \discord (with $m=k$). 
Finally, LSTM-AD marks as anomalies the subsequences that have the largest errors (distances) to the sequences that the algorithm predicts; we compute accuracy considering the subsequences with the $k$ largest errors.

In the first section of Table~\ref{scalabilityAccuracy}, we report the results of all techniques on the annotated real datasets with multiple (diverse and similar) anomalies. 
Series2Graph (both built on half and full dataset) is clearly the winner, with nearly perfect accuracy.
As expected, $Top$-$k$ $1^{st}$ \discord techniques (GV and STOMP) have in most of the cases lower accuracy than Series2Graph, since anomalies do not correspond to rare subsequences (i.e., isolated discords). 
We also observe that the $m^{th}$ discord technique (DAD), which is able to detect groups of $m$ similar anomalous subsequences, does not perform well. 
This is due to the many false positives produced by the algorithm. 

In the rest of Table~\ref{scalabilityAccuracy}, we report the accuracy of the evaluated methods on all the synthetic datasets (where we vary the number of anomalies, the \% of Gaussian noise and the anomaly subsequence length $\ell$). 
We note that the accuracy of the \discord discovery techniques substantially improves since in this case, most anomalies correspond to rare and isolated subsequences (i.e., different from one another).
Even in these cases though, Series2Graph is on average significantly more accurate than the competitors. 
Moreover, in contrast to GV, STOMP and DAD, we observe that Series2Graph's performance is stable as the noise increases between 0\%-25\%.

Regarding LSTM-AD, we note that, in general, it is more accurate than the \discord based algorithms. 
However, LSTM-AD cannot match the performance of Series2Graph, and in some cases it completely misses the anomalies (i.e., for the SED, MBA(806) and MBA(820) datasets).
Regarding LOF, we observe that it does not perform well. 
Isolation Forest on the other hand achieves a surprisingly good accuracy, which makes it a strong competitor.

Overall, we observe that regular Series2Graph ($S2G_{|T|}$) is considerably more accurate than all competitors (with rare exceptions, for which its performance is still very close to the best one), in all the settings we used in our evaluation. 

\subsection{Scalability Evaluation}

\begin{figure*}[tb]
  \centering
  \includegraphics[scale=0.83]{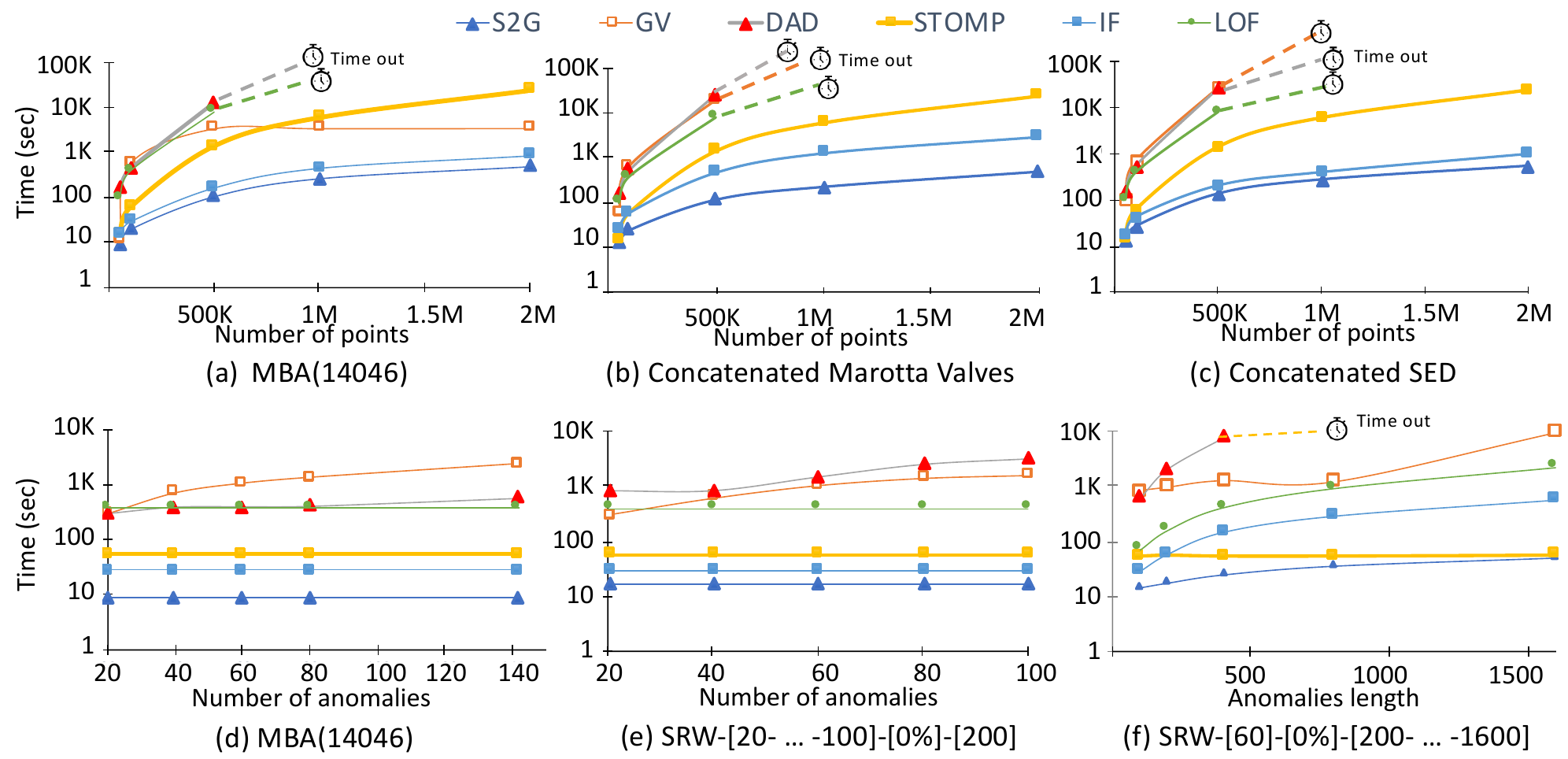}
  \vspace*{-0.4cm}
  \caption{Execution time vs data size (a-c), number of anomalies (d,e), anomaly length (f). Time out: 8h.}
  \label{fig:Scalability_exp}
\end{figure*}

We now present scalability tests (we do not consider LSTM-AD, since supervised methods have a completely different way of operation and associated costs, e.g., data labeling and model training). 
In Figures~\ref{fig:Scalability_exp}(a,b,c), we report the execution time of the algorithms (log scale) versus the size of the dataset. 
We use several prefix snippets (50K, 100K, 500K, 1M, 2M points) of the MBA(14046) dataset, a 2 million points concatenated version of the Marotta Valve dataset, and a 2 million points concatenated version of SED dataset. 
For all three datasets, $k$ is set to be equal to the number of anomalies in each snippet. 
We observe that Series2Graph is faster than the competitors, especially when both the input series length $\ell$ and anomaly subsequence length $\ell_A$ take large values, as in Figure~\ref{fig:Scalability_exp}(b), and gracefully scales with the dataset size. 

We also measure the execution time of the algorithms (log scale) as we vary the number of anomalies. 
We use the MBA(14046) dataset, as well as the synthetic datasets SRW-[20-100]-[0\%]-[200] 
(Figure \ref{fig:Scalability_exp}(e)). 
As expected, we observe that the performance of Series2Graph is not influenced by the number of anomalies. Similarly, STOMP and IF are not affected either, but GrammarViz, LOF and DAD are negatively impacted by the number of anomalies.

Finally, Figure~\ref{fig:Scalability_exp}(f) depicts the time performance results as we vary the length of the anomalies between 100-1600 points in the synthetic data series SRW-[60]-[0\%]-[100-1600]. 
The performance of STOMP is constant, because its complexity is not affected by the (anomaly) subsequence length. 
Moreover, we note that GV, IF, LOF and DAD perform poorly as the length of the anomalies is increasing. 
Observe that the execution time of Series2Graph increases slightly for larger subsequence lengths. 
This is due to the scoring function (last step of the algorithm). 
This function sums up all the edge weights of the subsequences we are interested in. 
Therefore, if the subsequence is large, the number of relevant edges is large as well, which slightly affects computation time. 
Nevertheless, Series2Graph remains the fastest algorithm among all competitors.

\section{Related Work}
\label{sec:related}

\noindent{\bf[Discord Discovery]} 
The problem of subsequence anomaly discovery has been studied by several works that use the $discord$ definition~\cite{DBLP:conf/icdm/YehZUBDDSMK16,DBLP:conf/edbt/Senin0WOGBCF15,Keogh2007,Liu2009,DBLP:conf/adma/FuLKL06,DBLP:conf/sdm/BuLFKPM07,Parameter-Free_Discord,valmodjournal}.
In these studies, anomalies are termed the isolated subsequences, i.e., the ones that have the highest Euclidean distances to their nearest neighbors. 
The proposed solutions either operate directly on the raw values~\cite{DBLP:conf/icdm/YehZUBDDSMK16,Liu2009,DBLP:conf/adma/FuLKL06,Parameter-Free_Discord}, or on discrete representations of the data, e.g., Haar wavelets~\cite{DBLP:conf/adma/FuLKL06,DBLP:conf/sdm/BuLFKPM07}, or SAX~\cite{Keogh2007,DBLP:conf/edbt/Senin0WOGBCF15}.
Even though the information that the discords carry is interesting and useful for some applications, these approaches (that are based on the $discord$ definition) fail when the dataset contains multiple anomalies that are similar to one another.

\noindent{\bf[Multiple \discord Discovery]}
The notion of $m^{th}$ \discord has been proposed in order to solve the issue of multiple similar anomalies~\cite{DBLP:journals/kais/YankovKR08}.
The approach described in this study finds the sequence that has the furthest $m^{th}$ nearest neighbor in Euclidean space.
During the search, a space pruning strategy based on the intermediate results of the simple $discord$ discovery is applied.
As we have already discussed, the $m^{th}$-$discord$ definition fixes the main problem of simple $discord$, but is very sensitive to the $m$ parameter and can lead to false positives. 

\noindent{\bf[Outlier Detection]} 
Local Outlier Factor~\cite{Breunig:2000:LID:342009.335388} is a degree of being an outlier assigned to a multidimensional data point. This degree depends on how much the data point is isolated (in terms of distance) to the surrounding neighborhood.
Similarly, Isolation Forest~\cite{Liu:2008:IF:1510528.1511387} is a classical machine learning technique that isolates anomalies instead of modeling normality. 
It first builds binary trees with random splitting nodes to partition the dataset. 
The anomaly score is defined as a function of the averaged path lengths between a particular sample and the roots of the trees. 
The above two methods are not specifically targeted to data series subsequences anomaly detection, which is reflected in the low accuracy they achieve in several of the datasets we tested.
 
\noindent{\bf [Deep Learning Approaches]} 
Subsequence anomaly detection has also been studied in the context of supervised deep learning techniques, with the use of Long Short Term Memory (LSTM) architectures~\cite{Hochreiter:1997:LSM:1246443.1246450}.
The studies that use this recurrent neural network are based on a forecasting model~\cite{LSTManomaly, DBLP:journals/corr/BontempsCML17}. 
First, the LSTM network is trained using the data segments that do \emph{not} contain anomalies. 
Then, the sequence is examined and the LSTM network is used to forecast the subsequent values: when the error between the forecast and the real value is above some threshold, the subsequence is classified as an anomaly. The system learns the threshold using the validation set, picking the value that maximizes the F1-score of the classification results. While the aforementioned approach originally used the annotated anomalies to learn the threshold, 
the LSTM model has also been used in a zero positive learning framework, where the annotated anomalies are not necessary for the training phase~\cite{DBLP:journals/corr/abs-1801-03168}. 

\noindent{\bf [Phase Space Reconstruction]} 
Phase space reconstruction is a technique that has been used for pattern embedding and for non-linear data series analysis~\cite{Kantz:2003:NTS:1121581, NonLinearRevisited}.
This technique transforms the data series into a set of vectors and has been used to visualize the evolution of the data series. 
Previous studies have also proposed the construction of a complex network based on the phase space reconstruction of a data series~\cite{ComplexNetwork, Gao_2016}, or on its visibility graph~\cite{Lacasa4972}, which can then be used to identify different patterns of interest. 
Series2Graph shares the same goal of converting the data series into a graph in order to reveal significant features. However, differently from the above methods that convert each point of the series into a separate node, Series2Graph uses a single node to represent several subsequences.



\section{Conclusions and Future Work}
\label{sec:concl}

Even though subsequence anomaly detection in data series has attracted a lot of attention, existing techniques have several shortcomings.
In this work, we describe a novel approach, based on a graph representation of the data series, which enables us to detect both single and recurrent anomalies (as well as the normal subsequences), in an unsupervised and domain-agnostic way. 
Experiments with several real datasets demonstrate the benefits of our approach in terms of both efficiency and accuracy. 
As future work, we plan to use modern data series indices~\cite{lernaeanhydra,lernaeanhydra2,parisplus,peng2020messi,evolutionofanindex,conf/sigmod/gogolou20,rtiskeynote} for accelerating the operation of Series2Graph, extend our approach to operate on streaming and multivariate data, and compare to the recently proposed  NorM approach~\cite{DBLP:conf/icde/BoniolLRP20a,DBLP:conf/icde/BoniolLRP20}.

\vspace*{0.3cm}
\noindent{\bf Acknowledgments:}
We thank M. Meftah and E. Remy, from EDF R\&D, for their support and advice, and the reviewers for the numerous constructive comments.
Work partially supported by EDF R\&D and ANRT French program.

\def\thebibliography#1{
  \section*{References}
    \vspace{-2pt}
	\normalsize
  \list
    {[\arabic{enumi}]}
    {\settowidth\labelwidth{[#1]}
     \leftmargin\labelwidth
     \parsep 0pt                
     \itemsep 0pt               
     \advance\leftmargin\labelsep
     \usecounter{enumi}
    }
  \def\newblock{\hskip .11em plus .33em minus .07em}
  \sloppy\clubpenalty10000\widowpenalty10000
  \sfcode`\.=1000\relax
}

\bibliographystyle{abbrv}
\bibliography{anomalies}

\begin{thebibliography}{10}

\bibitem{ourWebsite}
{Series2Graph Webpage}.
\newblock \url{http://helios.mi.parisdescartes.fr/~themisp/series2graph/},
  2020.

\bibitem{dD2019604}
D.~Abboud, M.~Elbadaoui, W.~Smith, and R.~Randall.
\newblock Advanced bearing diagnostics: A comparative study of two powerful
  approaches.
\newblock {\em MSSP}, 114, 2019.

\bibitem{doi:10.1177/1475921710395811}
A.~Abdul-Aziz, M.~R. Woike, N.~C. Oza, B.~L. Matthews, and J.~D. lekki.
\newblock Rotor health monitoring combining spin tests and data-driven anomaly
  detection methods.
\newblock {\em Structural Health Monitoring}, 2012.

\bibitem{doi:10.1117/12.847574}
M.~Ali Abdul-Aziz, N.~Woike, B.~Oza, Matthews, and G.~Baakilini.
\newblock Propulsion health monitoring of a turbine engine disk using spin test
  data, 2010.

\bibitem{IMSGroundtruth}
J.~Antoni and P.~Borghesani.
\newblock A statistical methodology for the design of condition indicators.
\newblock {\em Mechanical Systems and Signal Processing}, 2019.

\bibitem{DBLP:journals/dagstuhl-reports/BagnallCPZ19}
A.~J. Bagnall, R.~L. Cole, T.~Palpanas, and K.~Zoumpatianos.
\newblock Data series management (dagstuhl seminar 19282).
\newblock {\em Dagstuhl Reports}, 9(7), 2019.

\bibitem{31ffca1e82e94ed797d33e02a8a36dff}
S.~Bahaadini, V.~Noroozi, N.~Rohani, S.~Coughlin, M.~Zevin, J.~Smith,
  V.~Kalogera, and A.~Katsaggelos.
\newblock Machine learning for gravity spy: Glitch classification and dataset.
\newblock {\em Information Sciences}, 444:172--186, 5 2018.

\bibitem{statisticaloutliers}
V.~Barnet and T.~Lewis.
\newblock {\em {Outliers in Statistical Data}}.
\newblock {John Wiley and Sons, Inc.}, 1994.

\bibitem{DBLP:conf/icde/BoniolLRP20a}
P.~Boniol, M.~Linardi, F.~Roncallo, and T.~Palpanas.
\newblock Automated anomaly detection in large sequences.
\newblock In {\em {ICDE}}, 2020.

\bibitem{DBLP:conf/icde/BoniolLRP20}
P.~Boniol, M.~Linardi, F.~Roncallo, and T.~Palpanas.
\newblock {SAD:} an unsupervised system for subsequence anomaly detection.
\newblock In {\em {ICDE}}, 2020.

\bibitem{GraphAn}
P.~Boniol, T.~Palpanas, M.~Meftah, and E.~Remy.
\newblock Graphan: Graph-based subsequence anomaly detection.
\newblock {\em {PVLDB}}, 13(11), 2020.

\bibitem{DBLP:journals/corr/BontempsCML17}
L.~Bontemps, V.~L. Cao, J.~McDermott, and N.~Le{-}Khac.
\newblock Collective anomaly detection based on long short term memory
  recurrent neural network.
\newblock {\em CoRR}, abs/1703.09752, 2017.

\bibitem{NonLinearRevisited}
E.~Bradley and H.~Kantz.
\newblock Nonlinear time-series analysis revisited.
\newblock {\em Chaos: An Interdisciplinary Journal of Nonlinear Science}, 2015.

\bibitem{Breunig:2000:LID:342009.335388}
M.~M. Breunig, H.-P. Kriegel, R.~T. Ng, and J.~Sander.
\newblock Lof: Identifying density-based local outliers.
\newblock In {\em SIGMOD}, 2000.

\bibitem{DBLP:conf/sdm/BuLFKPM07}
Y.~Bu, O.~T. Leung, A.~W. Fu, E.~J. Keogh, J.~Pei, and S.~Meshkin.
\newblock {WAT:} finding top-k discords in time series database.
\newblock In {\em SIAM}, 2007.

\bibitem{DBLP:conf/kdd/ChiuKL03}
B.~Y. Chiu, E.~J. Keogh, and S.~Lonardi.
\newblock Probabilistic discovery of time series motifs.
\newblock In {\em SIGKDD 2003}, pages 493--498, 2003.

\bibitem{HelicopterPaperBigData2019}
N.~Daouayry, A.~Mechouche, P.-L. Maisonneuve, V.-M. Scuturici, and J.-M. Petit.
\newblock Data-centric helicopter failure anticipation: The mgb oil pressure
  virtual sensor case.
\newblock IEEE BigData, 2019.

\bibitem{rtiskeynote}
K.~Echihabi, K.~Zoumpatianos, and T.~Palpanas.
\newblock {Scalable Machine Learning on High-Dimensional Vectors: From Data
  Series to Deep Network Embeddings}.
\newblock In {\em {WIMS}}, 2020.

\bibitem{lernaeanhydra}
K.~Echihabi, K.~Zoumpatianos, T.~Palpanas, and H.~Benbrahim.
\newblock The lernaean hydra of data series similarity search: An experimental
  evaluation of the state of the art.
\newblock {\em {PVLDB}}, 2019.

\bibitem{lernaeanhydra2}
K.~Echihabi, K.~Zoumpatianos, T.~Palpanas, and H.~Benbrahim.
\newblock {Return of the Lernaean Hydra: Experimental Evaluation of Data Series
  Approximate Similarity Search}.
\newblock {\em {PVLDB}}, 2019.

\bibitem{Goldbergere215}
G.~et~al.
\newblock Physiobank, physiotoolkit, and physionet.
\newblock {\em Circulation}.

\bibitem{DBLP:conf/icdm/ZhuZSYFMBK16}
Y.~Z. et~al.
\newblock Matrix profile {II:} exploiting a novel algorithm and gpus to break
  the one hundred million barrier for time series motifs and joins.
\newblock In {\em {ICDM} 2016}.

\bibitem{DBLP:conf/adma/FuLKL06}
A.~W. Fu, O.~T. Leung, E.~J. Keogh, and J.~Lin.
\newblock Finding time series discords based on haar transform.
\newblock In {\em ADMA}, 2006.

\bibitem{ComplexNetwork}
Z.-K. Gao and N.~Jin.
\newblock Complex network from time series based on phase space reconstruction.
\newblock {\em Chaos (Woodbury, N.Y.)}, 19:033137, 09 2009.

\bibitem{Gao_2016}
Z.-K. Gao, M.~Small, and J.~Kurths.
\newblock Complex network analysis of time series.
\newblock {\em {EPL} (Europhysics Letters)}, 116(5):50001, dec 2016.

\bibitem{conf/sigmod/gogolou20}
A.~Gogolou, T.~Tsandilas, K.~Echihabi, A.~Bezerianos, and T.~Palpanas.
\newblock {Data Series Progressive Similarity Search with Probabilistic Quality
  Guarantees}.
\newblock In {\em SIGMOD}, 2020.

\bibitem{DBLP:conf/healthcom/HadjemNK16}
M.~Hadjem, F.~Na{\"{\i}}t{-}Abdesselam, and A.~A. Khokhar.
\newblock St-segment and t-wave anomalies prediction in an {ECG} data using
  rusboost.
\newblock In {\em Healthcom}, 2016.

\bibitem{SVDHALKO}
N.~Halko, P.~G. Martinsson, and J.~A. Tropp.
\newblock Finding structure with randomness: Probabilistic algorithms for
  constructing approximate matrix decompositions.
\newblock {\em SIAM Review}, 2011.

\bibitem{Hochreiter:1997:LSM:1246443.1246450}
S.~Hochreiter and J.~Schmidhuber.
\newblock Long short-term memory.
\newblock {\em Neural Comput.}, 9(8):1735--1780, Nov. 1997.

\bibitem{Kantz:2003:NTS:1121581}
H.~Kantz and T.~Schreiber.
\newblock {\em Nonlinear Time Series Analysis}.
\newblock Cambridge University Press, New York, NY, USA, 2003.

\bibitem{Keogh:2004:TPD:1014052.1014077}
E.~Keogh, S.~Lonardi, and C.~A. Ratanamahatana.
\newblock Towards parameter-free data mining.
\newblock In {\em Proceedings of the Tenth ACM SIGKDD International Conference
  on Knowledge Discovery and Data Mining}, KDD '04, pages 206--215, New York,
  NY, USA, 2004. ACM.

\bibitem{Keogh2007}
E.~Keogh, S.~Lonardi, C.~A. Ratanamahatana, L.~Wei, S.-H. Lee, and J.~Handley.
\newblock Compression-based data mining of sequential data.
\newblock {\em Data Mining and Knowledge Discovery}, 2007.

\bibitem{DBLP:conf/icdm/KeoghLF05}
E.~J. Keogh, J.~Lin, and A.~W. Fu.
\newblock {HOT} {SAX:} efficiently finding the most unusual time series
  subsequence.
\newblock In {\em ICDM}, 2005.

\bibitem{Lacasa4972}
L.~Lacasa, B.~Luque, F.~Ballesteros, J.~Luque, and J.~C. Nuno.
\newblock From time series to complex networks: The visibility graph.
\newblock {\em Proceedings of the National Academy of Sciences},
  105(13):4972--4975, 2008.

\bibitem{DBLP:journals/corr/abs-1801-03168}
T.~Lee, J.~Gottschlich, N.~Tatbul, E.~Metcalf, and S.~Zdonik.
\newblock Greenhouse: {A} zero-positive machine learning system for time-series
  anomaly detection.
\newblock {\em CoRR}, abs/1801.03168, 2018.

\bibitem{valmodjournal}
M.~Linardi, Y.~Zhu, T.~Palpanas, and E.~J. Keogh.
\newblock {Matrix Profile Goes MAD: Variable-Length Motif And Discord Discovery
  in Data Series}.
\newblock In {\em {DAMI}}, 2020.

\bibitem{Liu:2008:IF:1510528.1511387}
F.~T. Liu, K.~M. Ting, and Z.-H. Zhou.
\newblock Isolation forest.
\newblock In {\em ICDM}, ICDM, 2008.

\bibitem{Liu2009}
Y.~Liu, X.~Chen, and F.~Wang.
\newblock {Efficient Detection of Discords for Time Series Stream}.
\newblock {\em Advances in Data and Web Management}, pages 629--634, 2009.

\bibitem{Parameter-Free_Discord}
W.~Luo and M.~Gallagher.
\newblock Faster and parameter-free discord search in quasi-periodic time
  series.
\newblock In J.~Z. Huang, L.~Cao, and J.~Srivastava, editors, {\em Advances in
  Knowledge Discovery and Data Mining}, 2011.

\bibitem{LSTManomaly}
P.~Malhotra, L.~Vig, G.~Shroff, and P.~Agarwal.
\newblock Long short term memory networks for anomaly detection in time series.
\newblock 2015.

\bibitem{ThemisPaper2013}
K.~Mirylenka, A.~Marascu, T.~Palpanas, M.~Fehr, S.~Jank, G.~Welde, and
  D.~Groeber.
\newblock Envelope-based anomaly detection for high-speed manufacturing
  processes.
\newblock {\em European Advanced Process Control and Manufacturing Conference},
  2013.

\bibitem{Moody}
G.~B. Moody and R.~G. Mark.
\newblock The impact of the mit-bih arrhythmia database.
\newblock {\em IEEE Engineering in Medicine and Biology Magazine}, 2001.

\bibitem{DBLP:conf/sdm/MueenKZCW09}
A.~Mueen, E.~J. Keogh, Q.~Zhu, S.~Cash, and M.~B. Westover.
\newblock Exact discovery of time series motifs.
\newblock In {\em {SDM} 2009}.

\bibitem{Palpanas:2015:DSM:2814710.2814719}
T.~Palpanas.
\newblock Data series management: The road to big sequence analytics.
\newblock {\em SIGMOD Rec.}, 44(2):47--52, Aug. 2015.

\bibitem{evolutionofanindex}
T.~Palpanas.
\newblock {Evolution of a Data Series Index}.
\newblock {\em {CCIS}}, 1197, 2020.

\bibitem{itisareport}
T.~Palpanas and V.~Beckmann.
\newblock {Report on the First and Second Interdisciplinary Time Series
  Analysis Workshop (ITISA)}.
\newblock {\em {ACM SIGMOD Record}}, 48(3), 2019.

\bibitem{DBLP:journals/pvldb/PelkonenFCHMTV15}
T.~Pelkonen, S.~Franklin, P.~Cavallaro, Q.~Huang, J.~Meza, J.~Teller, and
  K.~Veeraraghavan.
\newblock Gorilla: {A} fast, scalable, in-memory time series database.
\newblock {\em {PVLDB}}, 8(12):1816--1827, 2015.

\bibitem{peng2020messi}
B.~Peng, P.~Fatourou, and T.~Palpanas.
\newblock {MESSI: In-Memory Data Series Indexing}.
\newblock In {\em {ICDE}}, 2020.

\bibitem{parisplus}
B.~Peng, T.~Palpanas, and P.~Fatourou.
\newblock Paris+: Data series indexing on multi-core architectures.
\newblock {\em TKDE}, 2020.

\bibitem{scott:1992}
D.~W. Scott.
\newblock {\em Multivariate Density Estimation. Theory, Practice, and
  Visualization}.
\newblock Wiley, 1992.

\bibitem{DBLP:conf/edbt/Senin0WOGBCF15}
P.~Senin, J.~Lin, X.~Wang, T.~Oates, S.~Gandhi, A.~P. Boedihardjo, C.~Chen, and
  S.~Frankenstein.
\newblock Time series anomaly discovery with grammar-based compression.
\newblock In {\em EDBT}, 2015.

\bibitem{DBLP:journals/tkdd/SeninLWOGBCF18}
P.~Senin, J.~Lin, X.~Wang, T.~Oates, S.~Gandhi, A.~P. Boedihardjo, C.~Chen, and
  S.~Frankenstein.
\newblock Grammarviz 3.0: Interactive discovery of variable-length time series
  patterns.
\newblock {\em {TKDD}}, 2018.

\bibitem{DBLP:conf/vldb/SubramaniamPPKG06}
S.~Subramaniam, T.~Palpanas, D.~Papadopoulos, V.~Kalogeraki, and D.~Gunopulos.
\newblock Online outlier detection in sensor data using non-parametric models.
\newblock In {\em VLDB 2006}, pages 187--198, 2006.

\bibitem{Wordrecognition}
J.~Wang, A.~Balasubramanian, L.~M. de~la Vega, J.~Green, A.~Samal, and
  B.~Prabhakaran.
\newblock Word recognition from continuous articulatory movement time-series
  data using symbolic representations.
\newblock In {\em SLPAT}.

\bibitem{DBLP:conf/icdm/WeiKX06}
L.~Wei, E.~J. Keogh, and X.~Xi.
\newblock Saxually explicit images: Finding unusual shapes.
\newblock In {\em Proceedings of the 6th {IEEE} International Conference on
  Data Mining {(ICDM} 2006), 18-22 December 2006, Hong Kong, China}, pages
  711--720, 2006.

\bibitem{Whitney}
C.~Whitney, D.~Gottlieb, S.~Redline, R.~Norman, R.~Dodge, E.~Shahar,
  S.~Surovec, and F.~Nieto.
\newblock Reliability of scoring respiratory disturbance indices and sleep
  staging.
\newblock {\em Sleep}, November 1998.

\bibitem{DBLP:conf/kdd/YankovKMCZ07}
D.~Yankov, E.~J. Keogh, J.~Medina, B.~Y. Chiu, and V.~B. Zordan.
\newblock Detecting time series motifs under uniform scaling.
\newblock In {\em ACM SIGKDD 2007}.

\bibitem{DBLP:conf/icdm/YankovKR07}
D.~Yankov, E.~J. Keogh, and U.~Rebbapragada.
\newblock Disk aware discord discovery: Finding unusual time series in terabyte
  sized datasets.
\newblock In {\em ICDM}, 2007.

\bibitem{DBLP:journals/kais/YankovKR08}
D.~Yankov, E.~J. Keogh, and U.~Rebbapragada.
\newblock Disk aware discord discovery: finding unusual time series in terabyte
  sized datasets.
\newblock {\em Knowl. Inf. Syst.}, 17(2):241--262, 2008.

\bibitem{DBLP:conf/icdm/YehZUBDDSMK16}
C.~M. Yeh, Y.~Zhu, L.~Ulanova, N.~Begum, Y.~Ding, H.~A. Dau, D.~F. Silva,
  A.~Mueen, and E.~J. Keogh.
\newblock Matrix profile {I:} all pairs similarity joins for time series: {A}
  unifying view that includes motifs, discords and shapelets.
\newblock In {\em ICDM}, pages 1317--1322, 2016.

\end{thebibliography}
\end{document}